\documentclass[letterpaper, 10pt, reqno]{amsart}
\usepackage{mathptmx}

\usepackage[usenames,dvipsnames,svgnames,table]{xcolor}
\usepackage{amsmath, amsthm, amssymb, bbm, bm, mathtools}
\usepackage{enumerate}
\usepackage{graphicx}
\usepackage{float}
\usepackage{wrapfig, subcaption}
\usepackage[margin=0cm]{caption}
\usepackage[titletoc, toc]{appendix}
\usepackage{tabularx}
\usepackage{booktabs,colortbl}
\usepackage{nicefrac}
\usepackage{blindtext}
\usepackage{marginnote}

\usepackage[boxruled, vlined, linesnumbered]{algorithm2e}
\SetAlFnt{\small}
\SetAlCapFnt{\small}
\SetAlCapNameFnt{\small}
\usepackage{algorithmic}
\algsetup{linenosize=\tiny}

\let\oldnl\nl
\newcommand{\nonl}{\renewcommand{\nl}{\let\nl\oldnl}}

\usepackage[bookmarks=false]{hyperref}
\hypersetup{
colorlinks=true, linkcolor=Sepia, citecolor=Sepia, filecolor=magenta, urlcolor=NavyBlue,
}
\usepackage{url, bookmark}

\newcommand{\reals}{\mathbb{R}}
\newcommand{\naturals}{\mathbb{N}}

\newcommand{\abs}[1]{\ensuremath \left| #1 \right|}
\newcommand{\normdouble}[1]{\ensuremath \lVert#1\rVert}
\newcommand{\norm}[1]{\ensuremath \left | #1\right |}
\newcommand{\given}{\, \vert \,}

\newcommand{\ag}[1]{\ensuremath \left\langle#1\right\rangle}
\providecommand{\OO}{\mathcal{O}}
\newcommand{\oo}{
 \mathchoice
  {{\scriptstyle\mathcal{O}}}
  {{\scriptstyle\mathcal{O}}}
  {{\scriptscriptstyle\mathcal{O}}}
  {\scalebox{.6}{$\scriptscriptstyle\mathcal{O}$}}
 }

\newcommand{\trace}{\trm{tr}}

\DeclareMathOperator*{\argmin}{arg\;min}

\newcommand{\aeq}[1]{\begin{align} #1 \end{align}}
\newcommand{\aeqs}[1]{\begin{align*} #1 \end{align*}}
\newcommand{\aed}[1]{\begin{aligned} #1 \end{aligned}}
\newcommand{\beq}[1]{\begin{equation}#1\end{equation}}
\newcommand{\beqs}[1]{\begin{equation*}#1\end{equation*}}

\newcommand{\trm}[1]{\mathrm{#1}}

\newcommand{\ilist}[1]{\begin{itemize}{#1}\end{itemize}}

\providecommand\f[2]{\ensuremath \frac{#1}{#2}}
\providecommand\rbrac[1]{\ensuremath \left(#1\right)}
\providecommand\sqbrac[1]{\ensuremath \left[#1\right]}
\providecommand\cbrac[1]{\ensuremath \left\{#1\right\}}
\DeclarePairedDelimiter{\ceil}{\lceil}{\rceil}

\newtheorem{theorem}{Theorem}

\newtheorem{lemma}[theorem]{Lemma}

\theoremstyle{definition}
\newtheorem{definition}[theorem]{Definition}
\newtheorem{example}[theorem]{Example}

\newtheorem{remark}[theorem]{Remark}

\newcommand{\E}{\mathbb{E}}

\providecommand{\ind}{{\bf 1}}

\renewcommand{\implies}{\Rightarrow}

\newcommand{\s}{\sigma}

\renewcommand{\r}{\rho}
\renewcommand{\t}{\tau}

\renewcommand{\a}{\alpha}

\newcommand{\e}{\epsilon}
\renewcommand{\b}{\beta}
\newcommand{\g}{\gamma}
\renewcommand{\d}{\delta}
\newcommand{\D}{\Delta}

\renewcommand{\L}{\Lambda}
\renewcommand{\l}{\lambda}

\renewcommand{\S}{\Sigma}

\def \CC {\mathcal{C}}

\def \LL {\mathcal{L}}

\def \OO {\mathcal{O}}

\newcommand{\data}{\trm{input}}
\newcommand{\convolution}{\trm{conv}}

\newcommand{\meanpool}{\textrm{mean-pool}}

\newcommand{\drop}{\trm{drop}}
\newcommand{\fc}{\trm{fc}}
\newcommand{\softmax}{\trm{softmax}}

\newcommand{\block}{\trm{block}}

\newcommand{\mnistfc}{\trm{mnistfc}}
\newcommand{\lenet}{\trm{LeNet}}
\newcommand{\allcnn}{\textrm{All-CNN}}

\newcommand{\lap}{\D}
\newcommand{\grad}{\nabla}
\newcommand{\R}{\reals}

\newcommand{\hmean}{\trm{HM}}
\usepackage{natbib}

\usepackage[margin=1.25in]{geometry}
\graphicspath{{./fig/}}

\newcommand{\ab}[2]{{\color{BurntOrange}#1}\marginnote{\tiny\noindent{\color{BurntOrange}[AB]\ #2}}}

\newcommand{\ignore}[1]{}

\title[PDEs for deep neural networks]{Deep Relaxation: partial differential equations\\ for optimizing deep neural networks}

\author[Chaudhari]{Pratik Chaudhari$^{1*}$}
\author[Oberman]{Adam Oberman$^{2*}$}
\author[Osher]{Stanley Osher$^3$}
\author[Soatto]{Stefano Soatto$^1$}
\author[Carlier]{Guillaume Carlier$^4$}
\thanks{$^*$ Joint first authors}

\begin{document}
\maketitle
{
\vspace*{-0.15in}
\footnotesize
\noindent $^{1}$ Computer Science Department, University of California, Los Angeles.\\
$^{2}$ Department of Mathematics and Statistics, McGill University, Montreal.\\
$^{3}$ Department of Mathematics \& Institute for Pure and Applied Mathematics, University of California, Los Angeles.\\
$^{4}$ CEREMADE, Universit\'e Paris IX Dauphine.\\[0.05in]
Email:\ \href{mailto:pratikac@ucla.edu}{pratikac@ucla.edu},
\href{mailto:adam.oberman@mcgill.ca}{adam.oberman@mcgill.ca},
\href{mailto:sjo@math.ucla.edu}{sjo@math.ucla.edu},
\href{mailto:soatto@ucla.edu}{soatto@ucla.edu},
\href{mailto:carlier@ceremade.dauphine.fr}{carlier@ceremade.dauphine.fr}\\
}

{\small
\noindent \textbf{\emph{Abstract:}}
In this paper we establish a connection between non-convex optimization methods for training deep neural networks and nonlinear partial differential equations (PDEs). Relaxation techniques arising in statistical physics which have already been used successfully in this context are reinterpreted as solutions of a viscous Hamilton-Jacobi PDE. Using a stochastic control interpretation allows we prove that the modified algorithm performs better in expectation that stochastic gradient descent.
Well-known PDE regularity results allow us to analyze the geometry of the relaxed energy landscape, confirming empirical evidence.
The PDE is derived from a stochastic homogenization problem, which arises in the implementation of the algorithm. The algorithms scale well in practice and can effectively tackle the high dimensionality of modern neural networks.\\
\vspace*{-0.05in}

{\footnotesize \noindent \emph{\textbf{Keywords:}} Non-convex optimization, deep learning, neural networks, partial differential equations, stochastic gradient descent, local entropy, regularization, smoothing, viscous Burgers, local convexity, optimal control, proximal, inf-convolution\\[-0.1in]}

\section{Introduction}
\label{s:intro}

\begin{figure}[!tbh]
\centering
\includegraphics[width=0.45\textwidth]{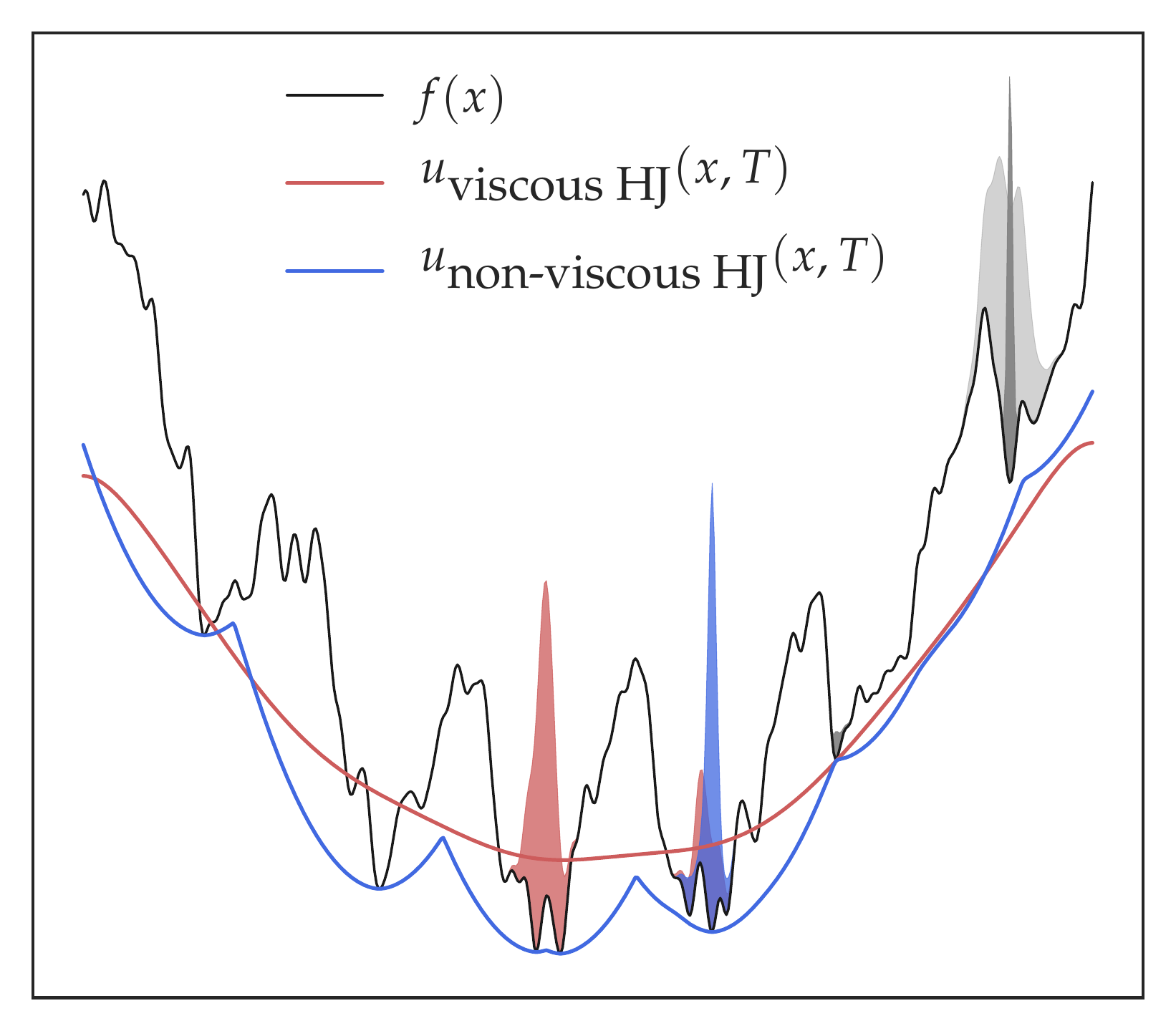}
\caption{\small One dimensional loss function $f(x)$ (black) and smoothing of $f(x)$ by the viscous Hamilton-Jacobi equation (red) and the non-viscous Hamilton Jacobi equation (blue).
The initial density (light grey) is evolved using the Fokker-Planck equation~\eqref{eq:mfg} which uses solutions of the respective PDEs. The terminal density obtained using SGD (dark gray) remains concentrated in a local minimum. The function $u_{\trm{viscous\ HJ}}(x,T)$ (red) is convex, with the global minimum close to that of $f(x)$, consequently the terminal density (red) is concentrated near the minimum of $f(x)$. The function $u_{\textrm{non-viscous\ HJ}}(x,T)$ (blue) is nonconvex, so the presence of gradients pointing away from the global minimum slows down the evolution. Nevertheless, the corresponding terminal density (blue) is concentrated in a nearly optimal local minimum. Note that $u_{\textrm{non-viscous\ HJ}}(x,T)$ can be interpreted as a generalized convex envelope of $f$: it is given by the Hopf-Lax formula~\eqref{eq:hopf_lax}. Also note that both the value and location of some local extrema for non-viscous HJ are unchanged, while the regions of convexity widen. On the other hand, local maxima located in smooth concave regions see their concave regions shrink, value eventually decreases. This figure was produced by solving~\eqref{eq:mfg} using monotone finite differences~\citep{ObermanSINUM}.
\vspace*{-0.1in}
}
\label{fig:smoothing}
\end{figure}

\subsection{Overview of the main results}
\label{ss:overview_results}
Deep neural networks have achieved remarkable success in a number of applied domains from visual recognition and speech, to natural language processing and robotics~\citep{lecun2015deep}. Despite many attempts, an understanding of the roots of this success remains elusive. Neural networks are a parametric class of functions whose parameters are found by minimizing a non-convex loss function, typically via stochastic gradient descent (SGD), using a variety of regularization techniques. In this paper, we study a modification of the SGD algorithm which leads to improvements in the training time and generalization error. 

\citet{baldassi2015subdominant} recently introduced the local entropy loss function, $f_\g(x)$, which is a modification of a given loss function $f(x)$, defined for $x \in \R^n$. \citet{chaudhari2016entropy} showed that the local entropy function is effective in training deep networks. Our first result, in Sec.~\ref{s:le_and_hjb}, shows that $f_\g(x)$ is the solution of a nonlinear partial differential equation (PDE). Define the viscous Hamilton-Jacobi PDE
\begin{align}
\label{eq:vhj}
\tag{viscous-HJ}
 \frac{\partial u}{\partial t} &= - \f{1}{2} \abs{\grad u}^2 + \f{\b^{-1}}{2}\lap u, \qquad \text{for } 0 < t \leq \gamma
 \end{align}
with $u(x,0) = f(x)$. Here we write $\Delta u = \sum_{i=1}^n\ \f{\partial^2}{\partial x_i^2} u$, for the Laplacian. We show that
\[
f_\g(x) = u(x,\g)
\]
Using standard semiconcavity estimates in Sec~\ref{s:semi_concavity}, the PDE interpretation immediately leads to regularity results for the solutions that are consistent with empirical observations.

Algorithmically, our starting point is the continuous time stochastic gradient descent equation,
\beq{
 dx(t) = -\grad f(x(t))\ dt + \b^{-1/2}\ dW(t)
 \label{eq:sde}
 \tag{SGD}
}
for $t \geq 0$ where $W(t) \in \reals^n$ is the $n$-dimensional Wiener process and the coefficient $\b^{-1}$ is the variance of the gradient term $\grad f(x)$ (see Section~\ref{ss:sgd} below).
\citet{chaudhari2016entropy} compute the gradient of local entropy $\grad f_\g(x)$ by first computing the invariant measure, $\r(y)$, of the auxiliary stochastic differential equation (SDE)
\begin{equation}
\label{eq:aux_sde}
 dy(s) = - \left (\g^{-1} y(s) + \grad f(x(t)-y(s)) \right) \ ds+ \b^{-1/2}\ dW(s)
\end{equation}
then by updating $x(t)$ by averaging against this invariant measure
\begin{equation}
 \label{eq:PDE_xbar}
 dx(t) = - \int \grad f\left(x(t) - y\right ) \rho(dy)
\end{equation}
The effectiveness of the algorithm may be partially explained by the fact that for small value of $\g$, that solutions of Fokker-Planck equation associated with the SDE \eqref{eq:aux_sde} converge exponentially quickly to the invariant measure $\rho$.
Tools from optimal transportation~\citep{santambrogio2015optimal} or functional inequalities~\citep{bakry1985diffusions} imply that there is a threshold parameter related to the semiconvexity of $f(x)$ for which the auxiliary SGD converges exponentially. See Sec.~\ref{ss:sgd_continuous} and Remark~\ref{rem:lambdac}. Beyond this threshold, the auxiliary problem can be slow to converge.

In Sec.~\ref{s:homogenization} we prove using homogenization of SDEs~\citep{pavliotis2008multiscale} that \eqref{eq:PDE_xbar} recovers the gradient of the local entropy function, $u(x,\g)$. In other words, in the limit of fast dynamics,~\eqref{eq:PDE_xbar} is equivalent to
\begin{equation}
 \label{xbar}
 dx(t) = -\grad u(x(t), \g)\ dt
\end{equation}
where $u(x,\g)$ is the solution of \eqref{eq:vhj}. Thus while smoothing the loss function using PDEs directly is not practical, due to the curse of dimensionality, the auxliary SDE \eqref{eq:aux_sde} accomplishes smoothing of the loss function $f(x)$ using exponentially convergent local dynamics.

The Elastic-SGD algorithm~\citep{zhang2015deep} is an influential algorithm which allowed training of Deep Neural Networks to be accomplished in parallel. We also prove, in Sec.~\ref{ss:esgd_local_entropy}, that the Elastic-SGD algorithm~\citep{zhang2015deep} is equivalent to regularization by local entropy. The equivalence is a surprising result, since the former runs by computing temporal averages on a single processor while the latter computes an average of coupled independent realizations of SGD over multiple processors. Once the homogenization framework is established for the local entropy algorithm, the result follows naturally from ergodicity, which allows us to replace temporal averages in~\eqref{eq:aux_sde} by the corresponding spatial averages used in Elastic-SGD.

Rather than fix one value of $\g$ for the evolution, \emph{scope} the variable. Scoping has the advantage of more smoothing earlier in the algorithm, while recovering the true gradients (and minima) of the original loss function $f(x)$ towards the end. Scoping is equivalent to replacing~\eqref{xbar} by
\[
 dx(t) = - \grad u(x(t), T-t)\ dt
\]
where $T$ is a termination time for the algorithm. We still fix $\g(t)$ in the inner loop represented by~\eqref{eq:aux_sde}. Scoping, which was previously considered to be an effective heuristic, is rigorously shown to be effective in Sec.~\ref{ss:stochastic_control}. Additionally, scoping corresponds to nonlinear forward-backward equations, which appear in Mean Field Games, see Remark~\ref{ss:mean_field_games}.

It is unusual in stochastic nonconvex optimization to obtain a proof of the superiority of a particular algorithm compared to standard stochastic gradient descent. We obtain such a result in Theorem~\ref{thm:improvement} of Sec.~\ref{ss:stochastic_control} Sec.~\ref{ss:stochastic_control} where we prove (for slightly modified dynamics) an improvement in the expected value of the loss function using the local entropy dynamics as compared to that of SGD. This result is obtained using well-known techniques from stochastic control theory~\citep{fleming2006controlled} and the comparison principle for viscosity solutions. Again, while these techniques are well-established in nonlinear PDEs, we are not aware of their application in this context.

The effectiveness of the PDE regularization is illustrated Figure~\ref{fig:smoothing}.

\section{Background}
\label{s:background}

\subsection{Deep neural networks}
\label{ss:deep_networks}
A deep network is a nested composition of linear functions of an input datum $\xi \in \reals^d$ and a nonlinear function $\s$ at each level. The output of a neural network, denoted by $y(x;\ \xi)$, can be written as
\beq{
 y(x;\ \xi) = \s' \rbrac{x^p\ \s \rbrac{ x^{p-2}\ \ldots \s \rbrac{x^1\ \xi}}\ldots };
 \label{eq:yhat}
}
where $x^1, \ldots, x^{p-1} \in \reals^{d \times d}$ and $x^p \in \reals^{d \times K}$ are the parameters or ``weights'' of the network. We will let $x \in \reals^n$ denote their concatenation after rearranging each one of them as a vector. The functions $\s(\cdot)$ are applied point-wise to each element of their argument. For instance, in a network with rectified linear units (ReLUs), we have $\s(z) = \max(0, z)$ for $z \in \reals^d$. Typically, the last non-linearity is set to $\s'(z) = \ind_{\cbrac{z\ \geq\ 0}}$. A neural network is thus a function with $n$ parameters that produces an output $y(x;\ \xi) \in \cbrac{1, \ldots, K}$ for each input $\xi$.

In supervised learning or ``training'' of a deep network for the task of image classification, given a finite sample of inputs and outputs $D = \{\xi^i, y^i)\}_{i = 1}^N$ (dataset), one wishes to minimize the empirical loss\footnote{The empirical loss is a sample approximation of the expected loss, $\E_{x \sim P} f(x)$, which cannot be computed since the data distribution $P$ is unknown. The extent to which the empirical loss (or a regularized version thereof) approximates the expected loss relates to generalization, i.e., the value of the loss function on (``test'' or ``validation'') data drawn from $P$ but not part of the training set $D$.}
\beq{
 f(x) := \f{1}{N}\ \sum_{i=1}^N\ f_i(x);
 \label{eq:fx}
}
where $f_i(x)$ is a loss function that measures the discrepancy between the predictions $y(x,\xi^i)$ for the sample $\xi^i$ and the true label $y^i$. For instance, the zero-one loss is
\[
 f_i(x) := \ind_{\cbrac{y(x;\ \xi^i)\ \neq\ y^i}}
\]
whereas for regression, the loss could be a quadratic penalty $f_i(x) := \f{1}{2} \normdouble{y(x;\ \xi^i) - y^i}^2_2$. Regardless of the choice of the loss function, \textbf{the overall objective $f(x)$ in deep learning is a non-convex function of its argument $x$}.

In the sequel, we will dispense with the specifics of the loss function and consider training a deep network as a generic non-convex optimization problem given by
\begin{equation}
	 x^* = \argmin_x\ f(x)
 \label{prob:generic_fx}
\end{equation}
However, the particular relaxation and regularization schemes we describe are developed for, and tested on, classification tasks using deep networks.

\subsection{Stochastic gradient descent (SGD)}
\label{ss:sgd}

First-order methods are the default choice for optimizing \eqref{prob:generic_fx}, since the dimensionality of $x$ is easily in the millions. Furthermore, typically, one cannot compute the entire gradient $N^{-1}\ \sum_{i=1}^N \grad f_i(x)$ at each iteration (update) because $N$ can also be in the millions for modern datasets.\footnote{For example, the ImageNet dataset~\citep{krizhevsky2012imagenet} has $N = 1.25$ million RGB images of size $224\times 224$ (i.e. $d \approx 10^5$) and $K=1000$ distinct classes. A typical model, e.g. the Inception network~\citep{szegedy2015going} used for classification on this dataset has about $N = 10$ million parameters and is trained by running~\eqref{eq:sgd} for $k \approx 10^5$ updates; this takes roughly $100$ hours with $8$ graphics processing units (GPUs).} Stochastic gradient descent~\citep{robbins1951stochastic} consists of performing partial computations of the form
\beq{
 x^{k+1} = x^k - \eta_k\ \nabla f_{i_k}(x^{k-1})\quad \trm{where}\ \eta_k > 0, \ k \in \naturals
 \label{eq:sgd}
}
where $i_k$ is sampled uniformly at random from the set $\cbrac{1,\ldots,N}$, at each iteration, starting from a random initial condition $x^0$. The stochastic nature of SGD arises from the approximation of the gradient using only a subset of data points. One can also average the gradient over a {\em set} of randomly chosen samples, called a ``mini-batch''. We denote the gradient on such a mini-batch by $\grad f_{\trm{mb}}(x)$.

If each $f_i(x)$ is convex and the gradient $\nabla f_i(x)$ is uniformly Lipschitz, SGD converges at a rate
\[
 \E\sqbrac{f(x^k)} - f(x^*) = \OO(k^{-1/2});
\]
this can be improved to $\OO(k^{-1})$ using additional assumptions such as strong convexity~\citep{nemirovski2009robust} or at the cost of memory, e.g., SAG~\citep{schmidt2013minimizing}, SAGA~\citep{defazio2014saga} among others. One can also obtain convergence rates that scale as $\OO(k^{-1/2})$ for non-convex loss functions if the stochastic gradient is unbiased with bounded variance~\citep{ghadimi2013stochastic}.

\subsection{SGD in continuous time}
\label{ss:sgd_continuous}

Developments in this paper hinge on interpreting SGD as a continuous-time stochastic differential equation. As is customary~\citep{ghadimi2013stochastic}, we assume that the stochastic gradient $\grad f_{\trm{mb}}(x^{k-1})$ in~\eqref{eq:sgd} is unbiased with respect to the full-gradient and has bounded variance, i.e., for all $x \in \reals^n$,
\aeqs{
 \E \sqbrac{\grad f_{\trm{mb}}(x)} &= \grad f(x),\\
 \E \sqbrac{\normdouble{\grad f_{\trm{mb}}(x) - \grad f(x)}^2} &\leq \b^{-1}_{\trm{mb}};
}
for some $\b_{\trm{mb}}^{-1} \geq 0$ We use the notation to emphasize that the noise is coming from the minibatch gradients. In the case where we compute full gradients, this coefficient is zero. For the sequel, we forget the source of the noise, and simply write $\b^{-1}$. The discrete-time dynamics in~\eqref{eq:sgd} can then be modelled by the stochastic differential equation~\eqref{eq:sde}
\[
 dx(t) = -\grad f(x(t))\ dt + \b^{-1/2}\ dW(t).
\]
The generator, $\LL$, corresponding to \eqref{eq:sde} is defined for smooth functions $\phi$ as
\beq{
 \label{eq:sde_generator}
 \LL \phi = -\grad f\ \cdot \grad\ \phi + \f{\b^{-1}}{2} \lap \phi.
}
The adjoint operator $\LL^*$ is given by $\LL^* \r = \grad\cdot(\grad f \r) + \f{\b^{-1}}{2} \lap\ \r$. Given the function $V(x): X \to \R$, define
\beq{
 \label{eq:value_sde}
 u(x,t) = \E \Big[ V(x(T)) ~\Big \vert~ x(t) = x \Big]
}
to be the expected value of $V$ at the final time for the path~\eqref{eq:sde} with initial data $x(t) = x$. By taking expectations in the It\^{o} formula, it can be established that $u$ satisfies the backward Kolmogorov equation
\aeqs{
 \label{eq:back_kolmogorov}
 \frac{\partial u}{\partial t} &= \LL u, \quad \trm{for}\ t < s \leq T,
}
along with the terminal value $u(x,T) = V(x)$. Furthermore, $\r(x, t)$, the probability density of $x(t)$ at time $t$, satisfies the Fokker-Planck equation~\citep{risken1984fokker}
\beq{
 \f{\partial}{\partial t}\ \r(x, t) = \grad \cdot \Big( \grad f(x)\ \r(x,t) \Big)\ + \f{\b^{-1}}{2}\ \lap \r(x,t)
 \label{eq:fp}
 \tag{FP}
}
along with the initial condition $\r(x,0) = \r_0(x)$, which represents the probability density of the initial distribution. With mild assumptions on $f(x)$, and even when $f$ is non-convex, $\rho(x,t)$ converges to the unique stationary solution of~\eqref{eq:fp} as $t \to \infty$ \cite[Section 4.5]{pavliotis2014stochastic}. Elementary calculations show that the stationary solution is the Gibbs distribution
 \beq{
 \r^\infty(x;\ \b) = Z(\b)^{-1} e^{-\b f(x)},
 \label{eq:gibbs}
}
for a normalizing constant $Z(\b)$. The notation for the parameter $\b > 0$ comes from physics, where it corresponds to the inverse temperature. From~\eqref{eq:gibbs}, as $\b \to \infty$, the Gibbs distribution $\r^\infty$ concentrates on the global minimizers of $f(x)$. In practice, ``momentum'' is often used to accelerate convergence to the invariant distribution; this corresponds to running the Langevin dynamics~\citep[Chapter 6]{pavliotis2014stochastic} given by
\[
 \aed{
 dx(t) &= p(t)\ dt\\
 dp(t) &= - \rbrac{\grad f(x(t)) + c\ p(t)}\ dt + \sqrt{c\ \b^{-1}}\ dW(t).
 }
\]
In the overdamped limit $c \to \infty$ we recover~\eqref{eq:sde} and in the underdamped limit $c = 0$ we recover Hamiltonian dynamics. However, while momentum is used in practice, there are more theoretical results available for the overdamped limit.

The celebrated paper~\citep{jordan1998variational} interpreted the Fokker-Planck equation as a gradient descent in the Wasserstein metric $d_{W_2}$~\citep{santambrogio2015optimal} of the energy functional
\[
 J(\r) = \int\ V(x)\ \r\ dx + \f{\b^{-1}}{2} \int\ \r\ \log \r\ dx;
\]
If $V(x)$ is $\l$-convex function $f(x)$, i.e., $\grad^2 V(x) +\l I$, is positive definite for all $x$, then the solution $\r(x,t)$ converges exponentially in the Wasserstein metric with rate $\lambda$ to $\r^\infty$~\citep{carrillo2006contractions},
\[
 d_{W_2} \rbrac{\r(x,t),\ \r^\infty} \leq d_{W_2} \rbrac{\r(x,0),\ \r^\infty}\ e^{-\l t}.
\]
Let us note that a less explicit, but more general, criterion for speed of convergence in the weighted $L^2$ norm is the Bakry-Emery criterion~\citep{bakry1985diffusions}, which uses PDE-based estimates of the spectral gap of the potential function. 

\subsection{Metastability and simulated annealing}
Under mild assumptions for non-convex functions $f(x)$, the Gibbs distribution $\r^\infty$ is still the unique steady solution of~\eqref{eq:fp}. However, convergence of $\r(x,t)$ can take an exponentially long time. Such dynamics are said to exhibit metastability: there can be multiple quasi-stationary measures which are stable on time scales of order one. Kramers' formula~\citep{kramers1940brownian} for Brownian motion in a double-well potential is the simplest example of such a phenomenon: if $f(x)$ is a double-well with two local minima at locations $x_1, x_2 \in \reals$ with a saddle point $x_3 \in \reals$ connecting them, we have
\[
 \E_{x_1} \sqbrac{\tau_{x_2}}\ \propto \f{1}{\abs{f^{''}(x_3)\ f^{''}(x_1)}^{1/2}} \exp \Big( \b (f(x_3) - f(x_1)) \Big);
\]
where $\t_{x_2}$ is the time required to transition from $x_1$ to $x_2$ under the dynamics in~\eqref{eq:sde}. The one dimensional example can be generalized to the higher dimensional case~\citep{bovier2006metastability}. Observe that there are two terms that contribute to the slow time scale: (i) the denominator involves the Hessian at a saddle point $x_3$, and (ii) the exponential dependence on the difference of the energies $f(x_1)$ and $f(x_2)$. In the higher dimensional case the first term can also go to infinity if the spectral gap goes to zero.

Simulated annealing~\citep{kushner1987asymptotic,chiang1987diffusion} is a popular technique which aims to evade metastability and locate the global minimum of a general non-convex $f(x)$. The key idea is to control the variance of the Wiener process, in particular, decrease it exponentially slowly by setting $\b = \b(t) = \log\ (t+1)$ in \eqref{eq:sgd}. There are theoretical results that prove that simulated annealing converges to the global minimum~\citep{geman1986diffusions}; however the time required is still exponential.

\subsection{SGD heuristics in deep learning}
\label{ss:sgd_practical}

State-of-the-art results in deep learning leverage upon a number of heuristics. Variance reduction of the stochastic gradients of Sec.~\ref{ss:sgd} leads to improved theoretical convergence rates for convex loss functions~\citep{schmidt2013minimizing,defazio2014saga} and is implemented in algorithms such as AdaGrad~\citep{duchi2011adaptive}, Adam~\citep{kingma2014adam} etc.

Adapting the step-size, i.e., changing $\eta_k$ in~\eqref{eq:sgd} with time $k$, is common practice. But while classical convergence analysis suggests step-size reduction of the form $\OO(k^{-1})$~\citep{robbins1951stochastic,schmidt2013minimizing}, staircase-like drops are more commonly used. Merely adapting the step-size is not equivalent to simulated annealing, which explicitly modulates the diffusion term. This can be achieved by modulating the mini-batch size $i_k$ in~\eqref{eq:sgd}. Note that as the mini-batch size goes to $N$ and the step-size $\eta_k$ is $\oo(N)$, the stochastic term in~\eqref{eq:sgd} vanishes~\citep{li2015dynamics}.

\subsection{Regularization of the loss function}
A number of prototypical models of deep neural networks have been used in the literature to analyze characteristics of the energy landscape. If one forgoes the nonlinearities, a deep network can be seen as a generalization of principal component analysis (PCA), i.e., as a matrix factorization problem; this is explored by~\citet{Baldi:1989:NNP:70359.70362, haeffele2015global,soudry2016no,DBLP:journals/corr/SaxeMG13}, among others. Based on empirical results such as~\citet{dauphin2014identifying}, authors in~\citet{Bray2007,Fyodorov2007,spinglass2015} and~\citet{chaudhari2015trivializing} have modeled the energy landscape of deep neural networks as a high-dimensional Gaussian random field. In spite of these analyses being drawn from vastly diverse areas of machine learning, they suggest that, in general, the energy landscape of deep networks is highly non-convex and rugged.

Smoothing is an effective way to improve the performance of optimization algorithms on such a landscape and it is our primary focus in this paper. This can be done by convolution with a kernel~\citep{chen2012smoothing}; for specific architectures this can be done analytically~\citep{mobahi2016training} or by averaging the gradient over random, local perturbations of the parameters~\citep{gulcehre2016noisy}.

In the next section, we present a technique for smoothing that has shown good empirical performance. We compare and contrast the above mentioned smoothing techniques in a unified mathematical framework in this paper.

\section{PDE interpretation of local entropy}
\label{s:le_and_hjb}
\label{sec:quadratic_expansion}

Local entropy is a modified loss function first introduced as a means for studying the energy landscape of the discrete perceptron, i.e., a ``shallow'' neural network with one layer and discrete parameters, e.g., $x \in \cbrac{-1,1}^n$~\citep{baldassi2016local}. An analysis of local entropy using the replica method~\citep{megard1987spin} predicts dense clusters of solutions to the optimization problem~\eqref{prob:generic_fx}. Parameters that lie within such clusters yield better error on samples outside the training set $D$, i.e. they have improved generalization error~\citep{baldassi2015subdominant}.

Recent papers by~\citet{baldassi2016unreasonable} and~\citet{chaudhari2016entropy} have suggested algorithmic methods to exploit the above observations. The later extended the notion of local entropy to continuous variables by replacing the loss function $f(x)$ with
\begin{equation}
 f_\g(x) := u(x,\g) = - \frac{1}{ \b} \log \Big( G_{\b^{-1} \g} *\ \exp \rbrac{- \b f(x)} \Big);
 \label{eq:local_entropy}
\end{equation}
where $G_\g(x) = {(2 \pi \g )^{-d/2}} \exp{-\f{\norm{x}^2}{2 \g}}$ is the heat kernel.

\subsection{Derivation of the viscous Hamilton-Jacobi PDE}
The Cole-Hopf transformation~\citep[Section 4.4.1]{Evansbook} is a classical tool used in PDEs. It relates solutions of the heat equation to those of the viscous Hamilton-Jacobi equation. For the convenience of the reader, we restate it here.
\begin{lemma}
\label{lem:cole_hopf}
\label{lem:grad_fg}
The local entropy function $f_\g(x) = u(x,\g)$ defined by~\eqref{eq:local_entropy} is the solution of the initial value problem for the viscous Hamilton-Jacobi equation \eqref{eq:vhj}
with initial values $u(x, 0) = f(x)$. Moreover, the gradient is given by
\begin{equation}
 \grad u(x,t) = \int_{\reals^n}\ \frac{y-x}{t} \r^\infty_1(d y;\ x) = \int_{\reals^n}\grad f(x-y)\ \r^\infty_2(dy;\ x)
\label{eq:grad_fg_exp}
\end{equation}
where $\r^\infty_i(y;\ x)$ are probability distributions given by
\begin{equation}
\label{eq:grad_fg_rho}
 \r^\infty_1(y;\ x) = Z_1^{-1} \exp \rbrac{- \b f(y)- \f{\b}{2 t} \norm{x-y}^2}
 \qquad
 \r^\infty_2(y; x) = Z_2^{-1} \exp \rbrac{- \b f(x-y) - \f{\b}{2 t} \norm{y}^2}
\end{equation}
and $Z_i = Z_i(x)$ are normalizing constants for $i=1,2$.
\end{lemma}

\begin{proof}
Define $u(x,t) = - {\beta^{-1}}\log v(x,t)$. From~\eqref{eq:local_entropy}, $v = \exp(- \beta u)$ solves the heat equation
\[
v_t = \beta^{-1} \lap v
\]
with initial data $v(x,0) = \exp(- \beta f(x))$.
Taking partial derivatives gives
\aeqs{
 v_t = - \b\ v\ u_t, &&
 \grad v = -\b\ v\ \grad u, &&
 \Delta v = -\b\ v\ \D u + \b^2\ v\ \abs{\grad u}^2.
}
Combining these expressions results in~\eqref{eq:vhj}.
Differentiating $v(x,t) = \exp(- \beta u(x,t))$ using the convolution in \eqref{eq:local_entropy}, gives up to positive constants which can be absorbed into the density,
\begin{align*}
\grad u(x,t)
 = C\ \grad_x \rbrac{G_t\ *\ e^{- \beta f(x)}}
 &= C\ \grad_x \int\ G_t\ (x-y)\ e^{- \beta f(y)}\ dy\\
 &= C\ \grad_x \int\ G_t\ (y)\ e^{-\beta f(x-y)}\ dy
\end{align*}
Evaluating the last or second to last expression for the integral leads to the corresponding parts of~\eqref{eq:grad_fg_exp} and \eqref{eq:grad_fg_rho}.
\end{proof}

\subsection{Hopf-Lax formula for the Hamilton-Jacobi equation and dynamics for the gradient}
\label{ss:hopf_lax_grad}
\label{s:hopf_lax}

In addition to the connection with~\eqref{eq:vhj} provided by Lemma~\ref{lem:cole_hopf}, we can also explore the non-viscous Hamilton-Jacobi equation. The Hamiliton-Jacobi equation corresponds to the limiting case of~\eqref{eq:vhj} as the viscosity term $\b^{-1} \to 0$.

There are several reasons for studying this equation.
It has a simple, explicit formula for the gradient. In addition, semi-concavity of the solution follows directly from the Hopf-Lax formula. Moreover, under certain convexity conditions, the gradient of the solution can be computed as an exponentially convergent gradient flow. The deterministic dynamics for the gradient are a special case of the stochastic dynamics for the viscous-HJ equation which are discussed in another section. Let us therefore consider the Hamilton-Jacobi equation
\beq{
\label{eq:hj}
\tag{HJ}
 \aed{
 u_t = -\f{1}{2}\ \abs{\grad u}^2}.
}
In the following lemma, we apply the well-known Hopf-Lax formula~\citep{Evansbook} for the solution of the HJ equation. It is also called the $\inf$-convolution of the functions $f(x)$ and $\f{1}{2t}\norm{x}^2$~\citep{cannarsa2004semiconcave}. It is closely related to the proximal operator~\citep{moreau1965proximite,rockafellar1976monotone}.
\begin{lemma}
\label{lem:grad_fg_hopf_lax}\label{lem:hopf_lax_dynamics}
Let $u(x,t)$ be the viscosity solution of \eqref{eq:hj} with $u(x,0) = f(x)$. Then
\beq{
 u(x,t) = \inf_y\ \cbrac{f(y) + \f{1}{2 t}\ \norm{x-y}^2}.
 \tag{HL}
 \label{eq:hopf_lax}
}
Define the proximal operator
\begin{equation}
 \label{prox}
 \trm{prox}_{tf}(x) = \argmin_y\left \{f(y) + \frac{1}{2t} \norm{x-y}^2 \right \}
\end{equation}
If $y^* = \textrm{prox}_{tf}(x)$ is a singleton, then $\grad_x u(x,t)$ exists, and
\beq{
 \grad_x\ u(x,t) = \f{x - y^*}{t} = \grad f(y^*),
 \label{eq:grad_fg_hopf_lax}
}
\end{lemma}

\begin{proof}
The Hopf-Lax formula for the solution of the HJ equation can be found in ~\citep{Evansbook}.
Danskin's theorem~\citep[Prop. 4.5.1]{bertsekas2003convex} allows us to pass a gradient through an infimum, by applying it at the argminimum.
The first equality is a direct application of Danskin's Theorem to \eqref{eq:hopf_lax}. For the second equality, rewrite~\eqref{eq:hopf_lax} with $z = (x-y)/t$ and the drop the prime, to obtain,
\[
 u(x,t) = \inf_{z}\ \cbrac{f(x-tz) + \f{t}{2}\ \norm{z}^2}.
\]
Applying Danskin's theorem to the last expresion gives
\[
 \grad u(x,t) = \grad f(x-tz^*) = \grad f(y^*).
\]
\end{proof}

Next we give a lemma which shows that under an auxiliary convexity condition, we can find $\grad(x,t)$ using exponentially convergent gradient dynamics.

\begin{lemma}[Dynamics for HJ]
\label{lem:HJ_dynamics}
For a fixed $x$, define
\[
 h(x,y;\ t) = f(x-ty) + \f{t}{2}\ \norm{y}^2.
\]
Suppose that $t > 0$ is chosen so that $h(x,y;t)$ is $\l$-convex as a function of $y$ (meaning that $D^2_y\ h(x,y,t) - \l I$ is positive definite). The gradient $p = \grad_x\ u(x,t)$ is then the unique steady solution of the dynamics
\beq{
 y'(s) = -\grad_y\ h(x,y(s);\ t) = -t \Big(y(s)-\grad f(x - t y(s)) \Big).
 \label{eq:p_hop_lax_min}
}
Moreover, the convergence to $p$ is exponential, i.e.,
\[
 \norm{y(s) - p} \leq \norm{y(0) - p}\ e^{-\l s}.
\]

\end{lemma}
\begin{proof}
Note that~\eqref{eq:p_hop_lax_min} is the gradient descent dynamics on $h$ which is $\l$-convex by assumption. Thus by standard techniques from ordinary differential equation (ODE) (see, for example~\citep{santambrogio2015optimal})
 the dynamics are contraction, and the convergence is exponential. \end{proof}

\subsection{The invariant measure for a quadratic function}
Local entropy computes a non-linear average of the gradient in the neighborhood of $x$ by weighing the gradients according to the steady-state measure $\r^\infty(y;\ x)$.

It is useful to compute $\grad f_\g(x)$ when $f(x)$ is quadratic, since this gives an estimate of the quasi-stable invariant measure near a local minimum of $f$. In this case, the invariant measure $\r^\infty(y;\ x)$ is a Gaussian distribution.

\begin{lemma}
\label{lem:rho_quadratic}
Suppose $f(x)$ is quadratic, with $p = \grad\ f(x)$ and $Q = \grad^2 f(x)$, then the invariant measure $\r^\infty_2(y;\ x)$ of~\eqref{eq:grad_fg_rho} is a normal distribution with mean $\mu$ and covariance $\S$ given by
\[
 \mu = x - \S\ p \quad \trm{and} \quad \S = \rbrac{Q + \g\ I}^{-1}.
\]
In particular, for $\g > \norm{Q}_2$, we have
\[
 \mu = x - \rbrac{\g^{-1}\ I - \g^{-2} Q} p \quad \trm{and} \quad \S = \g^{-1}\ I - \g^{-2} Q.
\]
\end{lemma}
\begin{proof}
Without loss of generality, a quadratic $f(x)$ centered at $x=0$ can be written as
\aeqs{
 f(x) &= f(0) + p\ x + \f{1}{2}\ x^\top Q x,\\
 \implies f(x) + \f{\g}{2}\ x^2 &= f(0) - \mu^\top p + \f{1}{2}\ \rbrac{y-\mu}^\top \rbrac{Q + \g\ I} \rbrac{y-\mu}
}
by completing the square. The distribution $\exp \rbrac{-f(y) - \f{1}{2 \g}\ \norm{y}^2}$ is thus a normal distribution with mean $\mu = x - \S\ p$ and variance $\S = \rbrac{Q + \g\ I}^{-1}$ and an appropriate normalization constant. The approximation $\S = \g^{-1}\ I - \g^{-2} Q$ which avoids inverting the Hessian follows from the Neumann series for the matrix inverse; it converges for $\g > \norm{Q}_2$ and is accurate up to $\OO\rbrac{\g^{-3}}$.
\end{proof}

\section{Derivation of local entropy via homogenization of SDEs}
\label{s:homogenization}

This section introduces the technique of homogenization for stochastic differential equations and gives a mathematical treatment to the result of Lemma~\ref{lem:grad_fg} which was missing in~\citet{chaudhari2016entropy}. In addition to this, homogenization allows us to rigorously show that an algorithm called Elastic-SGD~\citep{zhang2015deep} that was heuristically connected to a distributed version of local entropy in~\citet{baldassi2016unreasonable} is indeed equivalent to local entropy.

Homogenization is a technique used to analyze dynamical systems with multiple time-scales that have a few fast variables that may be coupled with other variables which evolve slowly. Computing averages over the fast variables allows us to obtain averaged equations for the slow variables in the limit that the times scales separate. We refer the reader to~\citet[Chap. 10, 17]{pavliotis2008multiscale} or~\citet{weinan2011principles} for details.

\subsection{Background on homogenization}
\label{s:homogenization_background}
Consider the system of SDEs given by
\beq{
\label{eq:homo_sde}
 \aed{
 dx(s) &= h(x,\ y)\ ds\\
 dy(s) &= \f{1}{\e}\ g(x,\ y)\ ds + \f{1}{\sqrt{\e \b }}\ dW(s);
 }
}
where $h, g$ are sufficiently smooth functions, $W(s) \in \reals^n$ is the standard Wiener process and $\e > 0$ is the homogenization parameter which introduces a fast time-scale for the dynamics of $y(s)$. Let us define the generator for the second equation to be
\[
 \LL_0 = g(x,y)\ \cdot\ \grad_y + \f{\b^{-1}}{2}\ \lap_y.
\]
We can assume that, for fixed $x$, the fast-dynamics, $y(s)$, has a unique invariant probability measure denoted by $\r^\infty(y;\ x)$ and the operator $\LL_0$ has a one-dimensional null-space characterized by
\[
 \LL_0\ 1(y) = 0, \quad \LL_0^*\ \r^\infty(y;\ x) = 0;
\]
here $1(y)$ are all constant functions in $y$ and $\LL_0^*$ is the adjoint of $\LL_0$. In that case, it follows that in the limit $\e \to 0$, the dynamics for $x(s)$ in~\eqref{eq:homo_sde} converge to
\[
 d X(s) = \overline{h}(X)\ ds
\]
where the homogenized vector field for $X$ is defined as the average agains the invariant measure. Moreover, by ergodicity, in the limit, the spatial average is equal to a long term average over the dynamics independent of the initial value $y(0)$.
\[
 \overline{h}(X) = \int\ h(X,y)\ \r^\infty(dy;\ X) = \lim_{T \to \infty}\ \f{1}{T}\ \int_0^T\ h(x, y(s))\ ds.
\]

\subsection{Derivation of local entropy via homogenization of SDEs}
\label{ss:le_derivation}
Recall \eqref{eq:grad_fg_exp} of Lemma~\ref{lem:grad_fg}
\[
 \grad f_\g(x) = -\g^{-1}\ \int_{\reals^n}\ (x-y)\ \r^\infty_1(y;\ x)\ dy.
\]
Hence, let us consider the following system of SDEs
\beq{
 \label{eq:le_sde_homo}
 \tag{Entropy-SGD}
 \aed{
  dx(s) &= -\g^{-1}\ (x-y)\ ds\\
  dy(s) &= -\f{1}{\e}\ \sqbrac{ \grad f(y) + \f{1}{\g}\ (y-x) }\ ds + \f{{\b^{-1/2}}}{\sqrt\e }\ dW(s).
 }
}
Write
\[
H(x,y;\g) = f(y) + \f{1}{2 \g} \norm{x-y}^2.
\]
The Fokker-Planck equation for the density of $y(s)$ is given by
\beq{
 \label{eq:fp_le_sgd}
 \r_t = \LL_0^*\ \r = \grad_y \cdot \rbrac{\grad_y\ H \r} + \f{\b^{-1}}{2}\ \lap_y\ \r;
}
The invariant measure for this Fokker-Planck equation is thus
\[
 \r^\infty_1(y;\ x) = Z^{-1} \exp \rbrac{-\b H(x,y;\g)}
\]
which agrees with~\eqref{eq:grad_fg_rho} in Lemma~\ref{lem:grad_fg}.
\begin{theorem}
\label{thm:le_homogenization}
As $\e \to 0$, the system~\eqref{eq:le_sde_homo} converges to the homogenized dynamics given by
\[
 d X(s) = -\grad f_\g(X)\ ds.
\]
Moreover, $-\grad f_\g(x) = -\g^{-1}\ \rbrac{x-\ag{y}}$ where
\beq{
 \ag{y} = \int\ y\ \r^\infty_1(dy;\ X) = \lim_{T\to \infty} \frac 1 T \int_0^T\ y(s)\ ds
 \label{eq:lesgd_ergodicity}
}
where $y(s)$ is the solution of the second equation in~\eqref{eq:le_sde_homo} for fixed $x$.
\end{theorem}
\begin{proof}
The result follows immediately from the convergence result in Sec.~\ref{s:homogenization_background} for the system~\eqref{eq:homo_sde}. The homogenized vector field is
\[
 \overline{h}(X) = -\g^{-1} \int\ (X-y)\ \r^\infty_1(y;\ X)
\]
which is equivalent to $\grad f_\g(X)$ from Lemma~\ref{lem:grad_fg}.
\end{proof}
By Lemma~\ref{lem:cole_hopf}, the following dynamics also converge to the gradient descent dynamics for $f_\g$.
 \beqs{
\label{eq:le_sde_homo_2}
\tag{Entropy-SGD-2}
 \aed{
  dx(s) &= -\grad f(x-y)\ ds\\
  dy(s) &= -\f{1}{\e}\ \sqbrac{\f{y}{\g} - \grad f(x-y)}\ ds + \f{1}{\sqrt{\e}}\ dW(s).
 }
}

\begin{remark}[Exponentially fast convergence]
\label{rem:lambdac}
Theorem~\ref{thm:le_homogenization} relies upon the existence of an ergodic, invariant measure $\r^\infty_1(y; x)$. Convergence to such a measure is exponentially fast if the underlying function, $H(x,y;\g)$, is convex in $y$; in our case, this happens if $\grad^2 f(x) + \g^{-1} I$ is positive definite for all $x$.
\end{remark}

\subsection{Elastic-SGD as local entropy}
\label{ss:esgd_local_entropy}

Elastic-SGD is an algorithm introduced by~\citet{zhang2015deep} for distributed training of deep neural networks and aims to minimize the communication overhead between a set of workers that together optimize replicated copies of the original function $f(x)$. For $n_p > 0$ distinct workers, Let $y = (y_1, \dots, y_{n_p})$, and let
\[
 \bar y = \f{1}{n_p} \sum_{j=1}^{n_p}\ y_j
\]
be the average value. Define the modified loss function
\beq{
 \min_{x}\ \frac{1}{n_p} \sum_{i=1}^{n_p} \left ( f(x_i) + \f{1}{2 \g}\ \norm{x_i - \bar x}^2 \right ).
 \label{eq:esgd_loss}
}
The formulation in~\eqref{eq:esgd_loss} lends itself easily to a distributed optimization where worker $i$ performs the following update, which depends only on the average of the other workers
\begin{equation}
 dy_i(s) = - \grad f(y_i(s)) - \frac{1}{\g} (y_i(s) - \bar y(s)) + \b^{-1/2} dW_i(s)
 \label{eq:each_worker}
\end{equation}
This algorithm was later shown to be connected to the local entropy loss $f_\g(x)$ by~\citet{baldassi2016unreasonable} using arguments from replica theory.
The results from the previous section can be modified to prove that the dynamics above converges to gradient descent of the local entropy. Consider the system
\[
 dx(s) = \f{1}{\g} \rbrac{x - \bar y}\ ds
\]
along with \eqref{eq:each_worker}.
As in Sec.~\ref{ss:le_derivation}, define $H(x,y;\g) = \sum_{i=1}^{n_p} f(y_i) + \f{1}{2 \g} \norm{x - \bar y }^2$. The invariant measure for the $y(s)$ dynamics corresponding to the $\e$ scaling of \eqref{eq:each_worker} is given by
\[
 \r^\infty(y;\ x) = Z^{-1} \exp \Big(-\b\ H(x,y;\ \g) \Big),
\]
and the homogenized vector field is
\[
 \overline{h}(X) = \g^{-1}\ \Big(X- \ag{\bar y} \Big)
\]
where $\ag{\bar y} = \int \bar y\ \r^\infty(d \bar y;\ X)$. By ergodicity, we can replace the average over the invariant measure with a combination of the temporal average and the average over the workers
\aeq{
 \label{eq:esgd_ergodicity}
 \ag{\bar y} &= \lim_{T \to \infty}\ \frac 1 T \int_0^T \bar y(s) \ ds
}
Thus the same argument as in Theorem~\ref{thm:le_homogenization} applies, and we can conclude that the dynamics also converge to gradient descent for local entropy.

\begin{remark}[Variational Interpretation of the distributed algorithm]
An alternate approach to understanding the distributed version of the algorithm, which we hope to study in the future, is modeling it as the gradient flow in the Wasserstein metric of a non-local functional. This functional is obtained by including an additional term in the Fokker-Planck functional,
\[
 J(\r) = \int\ f_\g (x) \r\ dx + \f{\b^{-1}}{2} \int\ \r\ \log \r\ dx + \f{1}{2\g}\ \int\ \int \norm{x-y}^2 \r(x)\ \r(y)\ dx \ dy;
\]
see \cite{santambrogio2015optimal}.
\end{remark}

\subsection{Heat equation versus the viscous Hamilton-Jacobi equation}
\label{ss:heat_vs_hjb}

Lemma~\ref{lem:cole_hopf} showed that local entropy is the solution to the viscous Hamilton-Jacobi equation~\eqref{eq:vhj}. The homogenization dynamics in~\eqref{eq:le_sde_homo} can thus be interpreted as a way of smoothing of the original loss $f(x)$ to aid optimization. Other partial differential equations could also be used to achieve a similar effect.
For instance, the following dynamics that corresponds to gradient descent on the solution of the the heat equation:
\beq{
\label{eq:heat_sde_homo}
\tag{HEAT}
 \aed{
  dx(s) &= -\grad f(x-y)\ ds\\
  dy(s) &= -\f{1}{\e \g}\ y\ ds + \f{1}{\sqrt{\e \b }}\ dW(s).
 }
}
Using a very similar analysis as above, the invariant distribution for the fast variables is $\r^\infty(dy;\ X) = G_{\b^{-1} \g}(y)$. In other words, the fast-dynamics for $y$ is decoupled from $x$ and $y(s)$ converges to a Gaussian distribution with mean zero and variance $\b^{-1} \g\ I$. The homogenized dynamics is given by
\beq{
 \label{eq:heat_ode_homo}
 d X(s) = -\grad f(X) *\ G_{\b^{-1} \g}\ ds = -\grad v(X,\g)
}
where $v(x,\g)$ is the solution of the heat equation $v_t = \frac{\b^{-1}} 2\lap v$ with initial data $v(x,0) = f(x)$. The dynamics corresponds to a Gaussian averaging of the gradients.

\begin{remark}[Comparison between local entropy and heat equation dynamics]
The extra term in the $y(s)$ dynamics for~\eqref{eq:le_sde_homo_2} with respect to~\eqref{eq:heat_sde_homo} is exactly the distinction between the smoothing performed by local entropy and the smoothing performed by heat equation. The latter is common in the deep learning literature under different forms, e.g.,~\citet{gulcehre2016noisy} and~\citet{mobahi2016training}. The former version however has much better empirical performance (cf.\@ experiments in Sec.~\ref{s:expts} and~\citet{chaudhari2016entropy}) as well as improved convergence in theory (cf.\@ Theorem~\ref{thm:improvement}). Moreover, at a critical point, the gradient term vanishes, and the operators~\eqref{eq:heat_sde_homo} and~\eqref{eq:le_sde_homo_2} coincide.
\end{remark}

\section{Stochastic control interpretation}
\label{ss:stochastic_control}
The interpretation of the local entropy function $f_\g(x)$ as the solution of a viscous Hamiltonian-Jacobi equation provided by Lemma~\ref{lem:cole_hopf} allows us to interpret the gradient descent dynamics as an optimal control problem. While the interpretation does not have immediate algorithmic implications, it allows us to prove that the expected value of a minimum is improved by the local entropy dynamics compared to the stochastic gradient descent dynamics. Controlled stochastic differential equations \citep{fleming2006controlled, fleming2012deterministic} are generalizations of SDEs discussed earlier in Sec.~\ref{ss:sgd_continuous}.

\subsection{Stochastic control for a variant of local entropy}
Consider the following controlled SDE
\beq{
 \label{eq:csgd}
 \tag{CSGD}
 \aed{
  dx(s) = -\grad f(x(s))\ ds - \a(s)\ ds + \b^{-1/2}\ dW(s), \quad \trm{for}\ t \leq s \leq T,
 }
}
along with $x(t) = x$. For fixed controls $\a$, the generator for~\eqref{eq:csgd} is given by
\[
 \LL_\a := (-\grad f - \a )\cdot \grad + \f{\b^{-1}}{2}\ \lap
\]
Define the cost functional for a given solution $x(\cdot)$ of \eqref{eq:csgd} corresponding to the control $\a(\cdot)$,
\beq{
 \label{eq:hjb_cost}
 \CC(x(\cdot),\ \a(\cdot)) = \E \sqbrac{ V(x(T)) + \f{1}{2}\ \int_0^T\ \norm{\a(s)}^2\ ds}.
}
Here the terminal cost is a given function $V: \reals^n \to \reals$ and we use the prototypical quadratic running cost. Define the value function
\[
 u(x,t) = \min_{\a(\cdot)}\ \CC(x(\cdot),\ \a(\cdot)).
\]
to be the minimum expected cost, over admissible controls, and over paths which start at $x(t) = x$. From the definition, the value function satisfies
\[
 u(x,T) = V(x).
\]
The dynamic programming principle expresses the value function as the unique viscosity solution of a Hamilton-Jacobi-Bellman (HJB) PDE~\citep{fleming2006controlled, fleming2012deterministic}, $u_t = H(\grad u,\ D^2 u)$ where the operator $H$ is defined by
\[
 H(p,Q) = \min_{\a} \cbrac{\LL_\a(p,Q) + \f{1}{2} |\a|^2} := \min_\alpha \cbrac{(-\grad f - \a )\cdot p + \f{1}{2}\ |\a|^2 + \f{\b^{-1}}{2}\ \trace\ Q}
\]
The minimum is achieved at $\alpha = p$, which results in $H(p,Q) = -\grad f\cdot p - \frac 1 2 |p|^2 + \f{\b^{-1}}{2}\ \trace\ Q$. The resulting PDE for the value function is
\beq{
 \label{eq:hjb}
 \tag{HJB}
 \aed{
 -u_t(x,t) &= -\grad f(x) \cdot \grad u(x,t) - \f{1}{2}\ \abs{\grad u(x,t)}^2 + \f{\b^{-1}}{2}\ \lap u(x,t), \qquad &\trm{for}\ t \leq s \leq T,
 }
}
along with the terminal values $u(x,T) = V(x)$. We make note that in this case, the optimal control is equal to the gradient of the solution
\beq{
 \a(x, t) = \grad u(x,t).
 \label{eq:alpha_grad_u}
}

\begin{remark}[Forward-backward equations]
\label{ss:mean_field_games}
Note that the PDE~\eqref{eq:hjb} is a terminal value problem in backwards time. This equation is well-posed, and by relabelling time, we can obtain an initial value problem in forward time. The corresponding forward equation for the evolution is of a density under the dynamics \eqref{eq:csgd}, involving the gradient of the value function. Together, these PDEs are the forward-backward system
\beq{
 \label{eq:mfg}
 \aed
 {
  -u_t &= -\grad f \cdot \grad u - \f{1}{2}\ \abs{\grad u}^2 + \f{\b^{-1}}{2}\ \lap u\\
  \r_t &= - \grad\ \cdot \Big( \grad u\ \rho \Big)+ \lap \r,
 }
}
for $0\ \leq\ s\ \leq\ T$ along with the terminal and the initial data $u(x,T) = V(x)$, $\r(x,0) = \r_0(x)$.
\end{remark}

\begin{remark}[Stochastic control for \eqref{eq:vhj}]
We can also obtain a stochastic control interpretation for \eqref{eq:vhj} by dropping the $\grad f(x(s))$ term in \eqref{eq:csgd}. Then, by a similar argument,
\eqref{eq:hjb} results in \eqref{eq:vhj}. Thus we obtain the interpretation of solutions of \eqref{eq:vhj} as the value function of an optimal control problem, minimizing the cost function \eqref{eq:hjb_cost} but with the modified dynamics. The reason for including the $\grad f$ term in \eqref{eq:csgd} is that it allows us to prove the comparison principle in the next subsection.
\end{remark}

\begin{remark}[Mean field games]
More generally, we can consider stochastic control problems where the control and the cost depend on the density of other players. In the case of ``small'' players, this can lead to mean field games~\citep{Lasry2007,huang2006large}. In the distributed setting is it natural to consider couplings (control and costs) which depend on the density (or moments of the density). In this context, it may be possible to use mean field game theory to prove an analogue of Theorem~\ref{thm:improvement} and obtain an estimate of the improvement in the terminal reward.
\end{remark}

\subsection{Improvement in the value function}
The optimal control interpretation above allows us to provide the following theorem for the improvement in the value function obtained by the dynamics \eqref{eq:csgd} using the optimal control \eqref{eq:alpha_grad_u}, compared to stochastic gradient descent.

\newcommand{\xcsgd}{x_{\trm{csgd}}}
\newcommand{\xsgd}{x_{\trm{sgd}}}
\begin{theorem}
\label{thm:improvement}
Let $\xcsgd(s)$ and $\xsgd(s)$ be solutions of~\eqref{eq:csgd} and~\eqref{eq:sde}, respectively, with the same initial data $\xcsgd(0) = \xsgd(0) = x_0$.
Fix a time $t \geq 0$ and a terminal function, $V(x)$. Then
\[
 \E \sqbrac{V(\xcsgd(t))} \leq \E \sqbrac{V(\xsgd(t))} - \f{1}{2}\ \E \sqbrac{\int_0^t\ \norm{\a(\xcsgd(s), s)}^2\ ds}.
\]
\end{theorem}
The optimal control is given by $\alpha(x,t) = \grad u(x,t)$, where $u(x,t)$ is the solution of~\eqref{eq:hjb} along with terminal data $u(x,T) = V(x)$.

\begin{proof}
Consider \eqref{eq:sde} and let $\LL = -\grad f(x) \cdot \grad + \f{1}{2}\ \lap$ to be the generator, as in \eqref{eq:sde_generator}.
The expected value function $v(x, t) = \E \sqbrac{V(x(t))}$ is the solution of the PDE $v_t = \LL v$.
Let $u(x, t)$ be the solution of~\eqref{eq:hjb}. From the PDE \eqref{eq:hjb}, we have
\[
 u_t = \LL u - \f{1}{2}\ \abs{\grad u}^2 \leq 0.
\]
We also have $u(x,T) = v(x,T) = V(x)$.
Use the maximum principle~\citep{Evansbook} to conclude
\beq{
 u(x,s) \leq v(x,s), \quad \trm{for\ all}\ t \leq s \leq T;
 \label{eq:sgd_le_comparison}
}
Use the definition to get
\[
 v(x,t) = \E \sqbrac{V(x(t))\ \given x(t) = x};
\]
the expectation is taken over the paths of~\eqref{eq:sde}. Similarly,
\[
 u(x,t) = \E \sqbrac{V(x(t)) + \f{1}{2}\ \int_0^t\ \norm{\a}^2\ ds}
\]
over paths of~\eqref{eq:csgd} with $x(t) = x$. Here $\a(\cdot)$ is the optimal control. The interpretation along with \eqref{eq:sgd_le_comparison} gives the desired result. Note that the second term on the right hand side in Theorem~\ref{thm:improvement} above can also be written as
\[
 \f{1}{2}\ \E \sqbrac{\int_0^t\ \norm{\grad u(\xcsgd(s), s)}^2\ ds}.
 \qedhere
\]
\end{proof}

\section{Regularization, widening and semi-concavity}
\label{s:semi_concavity}

The PDE approaches discussed in Sec.~\ref{s:le_and_hjb} result in a smoother loss function. We exploit the fact that our regularized loss function is the solution of the viscous HJ equation. This PDE allows us to apply standard semi-concavity estimates from PDE theory~\citep{Evansbook,cannarsa2004semiconcave} and quantify the amount of smoothing. Note that these estimates do not depend on the coefficient of viscosity so they apply to the HJ equation as well. Indeed, semi-concavity estimates apply to solutions of PDEs in a more general setting which includes~\eqref{eq:vhj} and~\eqref{eq:hj}.

We begin with special examples which illustrate the widening of local minima. In particular, we study the limiting case of the HJ equation, which has the property that local minima are preserved for short times. We then prove semiconcavity and establish a more sensitive form of semiconcavity using the harmonic mean. Finally, we explore the relationship between wider robust minima and the local entropy.

\subsection{Widening for solutions of the Hamilton-Jacobi PDE}

A function $f(x)$ is semi-concave with a constant $C$ if $f(x) - \f{C}{2}\ \norm{x}^2$ is concave, refer to~\cite{cannarsa2004semiconcave}.
Semiconcavity is a way of measuring the width of a local minimum: when a function is semiconcave with constant $C$, at a local minimum, no eigenvalue of the Hessian can be larger than $C$. The semiconcavity estimates below establish that near high curvature local minima widen faster than flat ones, when we evolve $f$ by \eqref{eq:vhj}.

It is illustrating to consider the following example.
\begin{example}\label{eq:parab}
Let $u(x,t)$ be the solution of ~\eqref{eq:vhj} with $u(x,0) = c\norm{x}^2/2$. Then
\[
 u(x,t) = \f{\norm{x}^2}{2(t+c^{-1})}\ + \frac{\b^{-1} n}{2} \log (t+c^{-1}).
\]
In particular, in the non-viscous case, $u(x,t) = \f{\norm{x}^2}{2(t+c^{-1})}$.
\end{example}
\begin{proof}
This follows from taking derivatives:
\[
 \grad u(x,t) = \frac{x}{t+c^{-1}},
 \qquad
 \frac{\abs{\grad u}^2}2 = \frac{\norm{x}^2} {2(t+c^{-1})^2},
 \qquad
 \lap u(x,t) = \frac{n}{t+1},
 \qquad
 u_t = -\frac{\norm{x}^2} {2(t+c^{-1})^2} + \frac{\b^{-1} n}{2} \frac{1} {t+c^{-1}}.
 \qedhere
\]
\end{proof}

In the previous example, the semiconcavity constant of $u(x,t)$ is $C(t) = 1/(c^{-1} + t)$ where $c$ measures the curvature of the initial data. This shows that for very large $c$, the improvement is very fast, for example $c = 10^8$ leads to $C(t = .01) \approx 10$, etc. In the sequel, we show how this result applies in general for both the viscous and the non-viscous cases. In this respect, for short times, solutions of~\eqref{eq:vhj} and~\eqref{eq:hj} widen faster than solutions of the heat equation.

\begin{example}[Slower rate for the heat equation]
In this example we show that the heat equation can result in a very slow improvement of the semiconcavity. Let $v_1(x)$ be the first eigenfunction of the heat equation, with corresponding eigenvalue $\l_1$, (for example $v_1(x) = \sin(x)$). In many cases $\l_1$ is of order 1. Then, since $v(x,t) = \exp(-\l_1 t)\ v_1(x)$ is a solution of the heat equation, the seminconcavity constant for $v(x,t)$ is $C(t) = c_0 \exp(-\l_1 t)$, where $c_0$ is the constant for $v_1(x)$. For short time, $C(t) \approx c_0 (1-\l_1 t)$ which corresponds to a much slower decrease in $C(t)$.
\end{example}

Another illustrative example is to study the widening of convex regions for the non-viscous Hamilton-Jacobi equation \eqref{eq:hj}. As can be seen in Figure~\ref{fig:smoothing}, solutions of the Hamilton-Jacobi have additional remarkable properties with respect to local minima. In particular, for small values of $t$ (which depend on the derivatives of the initial function) local minima are preserved. To motivate these properties, consider the following examples. We begin with the Burgers equation.

\begin{example}[Burgers equation in one dimension]
There is a well-known connection between Hamilton-Jacobi equations in one dimension and the Burgers equation~\citep{Evansbook} which is the prototypical one-dimensional conservation law. Consider the solution $u(x,t)$ of the non-viscous HJ equation~\eqref{eq:hj} with initial value $f(x)$ where $x \in \reals$. Let $p(x,t) = u_x(x,t)$. Then $p$ satisfies the Burgers equation
\[
 p_t = -p\ p_x
\]
We can solve Burgers equation using the method of characteristics, for a short time, depending on the regularity of $f(x)$. Set $p(x,t) = u_x(x,t)$. For short time, we can also use the fact that $p(x,t)$ is the fixed point of~\eqref{eq:p_hop_lax_min}, which leads to
\begin{equation}
	\label{peq}
	 p(x,t) = f'(x - tp(x,t))
\end{equation}
In particular, the solutions of this equation remain classical until the characteristics intersect, and this time, $t^*$, is determined by
\begin{equation}
	\label{teq}
	t^* f''(x - p t^*) +1 = 0
\end{equation}
which recovers the convexity condition of Lemma~\ref{lem:HJ_dynamics}.
\end{example}

\begin{example}[Widening of convex regions in one dimension]
In this example, we show that the widening of convex regions for one dimensional displayed in Figure~\ref{fig:smoothing} occurs in general. Let $x$ be a local minimum of the smooth function $f(x)$, and let $0 \leq t < t^*$, where the critical $t^*$ is defined by \eqref{teq}. Define the convexity interval as:
\[
 I(x,t) = \text{ the widest interval containing $x$ where $u(x,t)$ is convex }
\]
Let $I(x,0) = [x_0,\ x_1]$, so that $f''(x) \geq 0$ on $[x_0,\ x_1]$ and $f''(x_0) = f''(x_1) = 0$. Then $I(x,t)$ contains $I(x,0)$ for $0 \leq t \leq t^*$.
Differentiating \eqref{peq}, and the solving for $p_x(x,t)$ leads to
\[
 \aed{
 p_x\ &= f''(x - t p)\ (1 - t p_x)\\
 \implies p_x &= \f{f''(x - t p)}{1+t\ f''(x- t p)}.
 }
\]
Since $t \leq t^*$, the last equation implies that $p_x$ has the same sign as $f''(x - t p)$.
Recall now that $x_1$ is an inflection point of $f(x)$. Choose $x$ such that $x_1 = x - tp(x,t)$, in other words, $u_{xx}(x,t) = 0$. Note that $x > x_1$, since $p(x,t) > 0$. Thus the inflection point moves to the right of $x_1$. Similarly the inflection point $x_0$ moves to the left. So we have established that the interval is larger.
\end{example}

\begin{remark}
In the higher dimensional case, well-known properties of the inf-convolution can be used to establish the following facts. If $f(x)$ is convex, $u(x,t)$ given by~\eqref{eq:hopf_lax} is also convex; thus, any minimizer of $f$ also minimizes $u(x,t)$. Moreover, for bounded $f(x)$, using the fact that $\norm{y^*(x) - x} = \OO(\sqrt{t})$, one can obtain a local version of the same result, viz., for short times $t$, a local minimum persists. Similar to the example in Figure~\ref{fig:smoothing}, local minima in the solution vanish for longer times.
\end{remark}

\subsection{Semiconcavity estimates}
\label{ss:bounds_hessian}
In this section we establish semiconcavity estimates. These are well known, and are included for the convenience of the reader.
\begin{lemma}
\label{lem:semi_concave}
Suppose $u(x,t)$ is the solution of~\eqref{eq:vhj}, and let $\b^{-1} \geq 0$. If
\[
 C_k = \sup_x\ u_{x_k x_k}(x,0) \quad \trm{and} \quad
 C_{\trm{Lap}}\ = \sup_x\ \lap u(x,0),
\]
then
\[
 \sup_x\ u_{x_k x_k}(x,t) \leq \f{1}{C_k^{-1} + t},
 \qquad \text{ and } \qquad
 \sup_x\ \lap u(x,t) \leq \f{1}{C_\trm{Lap}^{-1} + t/n}.
\]
\end{lemma}
\begin{proof}
1. First consider the nonviscous case $\b^{-1}= 0$. Define $g(x, y) = f(y) + \f{1}{2 t}\ \norm{x-y}^2$ to be a function of $x$ parametrized by $y$. Each $g(x, y)$, considered as a function of $x$ is semi-concave with constant $1/t$. Then $u(x,t)$, the solution of \eqref{eq:hj}, which is given by the Hopf-Lax formula \eqref{eq:hopf_lax} is expressed as the infimum of such functions. Thus $u(x,t)$ is also semi-concave with the same constant, $1/t$.

2. Next, consider the viscous case $\b^{-1} > 0$.
Let $w = u_{x_k x_k}$, differentiating~\eqref{eq:vhj} twice gives
\beq{
 w_t + \grad u\ \cdot\ \grad w + \sum_{i=1}^n\ u_{x_i x_k}^2 = \f{\b^{-1}}{2}\ \lap w.
 \label{eq:vhj_diff}
}
Using $u_{x_k x_k}^2 = w^2$, we have
\[
 w_t + \grad u \cdot \grad w - \f{\b^{-1}}{2}\ \lap w \leq -w^2.
\]
Note that $w(x,0) \leq C_k$ and the function
\[
 \phi(x,t) = \f{1}{C_k^{-1} + t}
\]
is a spatially independent super-solution of the preceding equation with $w(x,0) \leq \phi(x,0)$. By the comparison principle applied to $w$ and $\phi$, we have
\[
 w(x,t) = u_{x_k x_k}(x,t) \leq \f{1}{C_k^{-1} + t}
\]
for all $x$ and for all $t \geq 0$, which gives the first result.
Now set $v = \lap u$ and sum~\eqref{eq:vhj_diff} over all $k$ to obtain
\aeqs{
 v_t + \grad u \cdot \grad v + \sum_{i,j=1}^n\ u_{x_i x_j}^2 &= \f{\b^{-1}}{2}\ \lap v}
Thus
\[
 v_t + \grad u \cdot \grad v - \f{\b^{-1}}{2} \lap v \leq -\f{1}{n}\ v^2
\]
by the Cauchy-Schwartz inequality. Now apply the comparison principle to $v'(x,t) = \rbrac{C_{\trm{Lap}}^{-1} + t/n}^{-1}$ to obtain the second result.
\end{proof}

\subsection{Estimates on the Harmonic mean of the spectrum}
\label{sss:harmonic_mean}

Our semi-concavity estimates in Sec.~\ref{ss:bounds_hessian} gave bounds on $\sup_x\ u_{x_k x_k}(x,t)$ and the Laplacian $\sup_x\ \lap u(x,t)$. In this section, we extend the above estimates to characterize the eigenvalue spectrum more precisely. Our approach is motivated by experimental evidence in~\citet[Figure\@ 1]{chaudhari2016entropy} and~\citet{sagun2016singularity}. The authors found that the Hessian of a typical deep network at a local minimum discovered by SGD has a very large proportion of its eigenvalues that are close to zero. This trend was observed for a variety of neural network architectures, datasets and dimensionality.
For such a Hessian, instead of bounding the largest eigenvalue or the trace (which is also the Laplacian), we obtain estimates of the harmonic mean of the eigenvalues. This effectively discounts large eigenvalues and we get an improved estimate of the ones close to zero, that dictate the width of a minimum. The harmonic mean of a vector $x \in \reals^n$ is
\[
 \hmean(x) = \rbrac{\f{1}{n}\ \sum_{i=1}^n\ \f{1}{x_i} }^{-1}.
\]
The harmonic mean is more sensitive to small values as seen by
\[
 \min_i\ x_i \leq \hmean(x)\ \leq\ n \min_i\ x_i.
\]
which does not hold for the arithmetic mean. To give an example, the eigenspectrum of a typical network~\citep[Figure\@ 1]{chaudhari2016entropy} has an arithmetic mean of $0.0029$ while its harmonic mean is $\approx 10^{-10}$ which better reflects the large proportion of almost-zero eigenvalues.

\begin{lemma}
\label{lem:eig_diag}
If $\L$ is the vector of eigenvalues of the symmetric positive definite matrix $A \in \reals^{n\times n}$ and $D = (a_{11}, \ldots, a_{nn})$ is its diagonal vector,
\[
 \hmean(\L) \leq \hmean(D).
\]
\end{lemma}
\begin{proof}
Majorization is a partial order on vectors with non-negative components~\citep[Chap.\@ 3]{marshall1979inequalities} and for two vectors $x,y \in \reals^n$, we say that $x$ is majorized by $y$ and write $x \preceq y$ if $x = S y$ for some doubly stochastic matrix $S$. The diagonal of a symmetric, positive semi-definite matrix is majorized by the vector of its eigenvalues~\citep{schur1923uber}. For any Schur-concave function $f(x)$, if $x \preceq y$ we have
\[
 f(y) \leq f(x)
\]
from~\citet[Chap.\@ 9]{marshall1979inequalities}. The harmonic mean is also a Schur-concave function which can be checked using the Schur-Ostrowski criterion
\[
 (x_i - x_j) \rbrac{\f{\partial\ \hmean}{\partial x_i} -\f{\partial\ \hmean}{\partial x_j}} \leq 0 \quad \trm{for\ all}\ x \in \reals^n_+
\]
and we therefore have
\[
 \hmean(\L) \leq \hmean(D).
 \qedhere
\]
\end{proof}

\begin{lemma}
\label{lem:hmean_eig}
If $u(x,t)$ is a solution of~\eqref{eq:vhj} and $C \in \reals^n$ is such that $\sup_x\ u_{x_k x_k}(x,t) \leq C_k$ for all $k \leq n$, then at a local minimum $x^*$ of $u(x,t)$,
\[
 \hmean(\L) \leq \f{1}{t + \hmean(C)^{-1}}
\]
where $\L$ is the vector of eigenvalues of $\grad^2 u(x^*, t)$.
\end{lemma}
\begin{proof}
For a local minimum $x^*$, we have $\hmean(L) \leq \hmean(D)$ from the previous lemma; here $D$ is the diagonal of the Hessian $\grad^2 u(x_{\trm{min}},t)$. From Lemma~\ref{lem:semi_concave} we have
\[
 u_{x_k x_k}(x^*, t) \leq C_k(t) = \f{1}{t + C_k^{-1}}.
\]
The result follows by observing that
\[
\hmean(C(t)) = \f{n}{\sum_{i=1}^n\ C_k(t)^{-1}} = \f{1}{t + \hmean(C)^{-1}}.
\qedhere
\]
\end{proof}

\section{Algorithmic details}
\label{s:algorithmic_details}
In this section, we compare and contrast the various algorithms in this paper and provide implementation details that are used for the empirical validation are provided in Sec.~\ref{s:expts}.
\subsection{Discretization of \texorpdfstring{\eqref{eq:le_sde_homo}}{Entropy-SGD}}
\label{ss:esgd_algorithmic}
The time derivative in the equation for $y(s)$ in \eqref{eq:le_sde_homo} is discretized using the Euler-Maruyama method with time step (learning rate) $\eta_y$ as
\begin{equation}
	\label{eq:first_discrete_y}
	 y^{k+1} = y^k - \eta_y\ \sqbrac{\grad f(y^k) + \f{y^k-x^k}{\g}} + \sqrt{\eta_y\ \b^{-1}}\ \e^k
\end{equation}
where $\e_k$ are zero-mean, unit variance Gaussian random variables.

In practice, we do not have access to the full gradient $\grad f(x)$ for a deep network and instead, only have the gradient over a mini-batch $\grad f_{\trm{mb}}(x)$ (cf.\@ Sec.~\ref{ss:sgd}). In this case, there are two potential sources of noise: the coefficient $\b^{-1}_{\trm{mb}}$ arising from the mini-batch gradient, and any amount of extrinsic noise, with coefficient $\b^{-1}_{\trm{ex}}$. Combining these two sources leads to the equation
\beq{
 \label{eq:esgd_discrete_y}
 y^{k+1} = y^k - \eta_y\ \sqbrac{\grad f_{\trm{mb}}(y^k) + \f{y^k-x^k}{\g}} + \sqrt{\eta_y\ \b^{-1}_{\trm{ex}}}\ \e^k.
}
which corresponds to~\eqref{eq:first_discrete_y} with
\[
 \b^{-1/2} = \b^{-1/2}_{\trm{mb}} + \b^{-1/2}_{\trm{ex}}
\]
We initialize $y^k = x^k$ if $k/L$ is an integer. The number of time steps taken for the $y$ variable is set to be $L$. This corresponds to $\e = 1/L$ in~\eqref{eq:homo_sde}
and to $T=1$ in~\eqref{eq:lesgd_ergodicity}.
This results in a discretization of the equation for $x$ in \eqref{eq:le_sde_homo}
given by
\beq{
 \label{eq:esgd_discrete_x}
 \aed{
 x^{k+1} &= \begin{cases}
 x^k -\eta\ \g^{-1}\ \rbrac{x^k - \ag{y}^k}& \trm{if}\ k/L\ \trm{is\ an\ integer},\\
 x^k& \trm{otherwise};\\
 \end{cases}
 }
}
Since we may be far from the limit $\e \to 0$, $T \to \infty$, instead of using the linear averaging for $\ag{y}$ in \eqref{eq:lesgd_ergodicity}, we will use a forward looking average. Such an averaging gives more weight to the later steps which is beneficial for large values of $\g$ when the invariant measure $\r^\infty(y;\ x)$ in~\eqref{eq:lesgd_ergodicity} does not converge quickly.
The two algorithms we describe will differ in the choices of $L$ and $\b^{-1}_{\trm{ex}}$ and the definition of $\ag{y}$. For \textbf{Entropy-SGD}, we set $L = 20$, $\b^{-1}_{\trm{ex}} = 10^{8}$, and set $\ag{y}$ to be
\[
 \ag{y}^{k+1} = \a \ag{y}^k + (1-\a)\ y^{k+1};
\]
$\ag{y}^k$ is re-initialized to $x^k$ if $k/L$ is an integer. The parameter $\a$ is used to perform an exponential averaging of the iterates $y^k$. We set $\a=0.75$ for experiments in Sec.~\ref{s:expts} using Entropy-SGD. This is equivalent to the original implementation of~\citet{chaudhari2016entropy}.

The step-size for the $y^k$ dynamics is fixed to $\eta_y = 0.1$. It is a common practice to perform step-size annealing while optimizing deep networks (cf.\@ Sec.~\ref{ss:sgd_practical}), i.e., we set $\eta=0.1$ and reduce it by a factor of $5$ after every $3$ epochs (iterations on the entire dataset).

\subsection{Non-viscous Hamilton-Jacobi (HJ)}
\label{ss:hj_algorithmic}

For the other algorithm which we denote as \textbf{HJ}, we set $L = 5$ and $\b^{-1}_{\trm{ex}} = 0$, and set $\ag{y}^k = y^k$ in~\eqref{eq:esgd_discrete_y} and~\eqref{eq:esgd_discrete_x}, i.e., no averaging is performed and we simply take the last iterate of the $y^k$ updates.

\begin{remark}
We can also construct an equivalent system of updates for \textbf{HJ} by discretizing~\eqref{eq:le_sde_homo_2} and again setting $\b^{-1}_{\trm{ex}} = 0$; this exploits the two different ways of computing the gradient in Lemma~\ref{lem:cole_hopf} and gives
\beq{
 \label{eq:hj_discrete}
 \aed{
  y^{k+1} &= \rbrac{1 - \g^{-1}\ \eta_y}\ y^k + \eta_y\ \grad f_{\trm{mb}}(x^k - y^k)\\
  x^{k+1} &=
  \begin{cases}
   x^k -\eta\ \grad f_{\trm{mb}}(x^k - y^{k})& \trm{if}\ k/L\ \trm{is\ an\ integer},\\
   x^k & \trm{else};\\
  \end{cases}
 }
}
and we initialize $y^k = 0$ if $k/L$ is an integer.
\end{remark}

\subsection{Heat equation: }
\label{ss:heat_algorithmic}
We perform the update corresponding to~\eqref{eq:heat_ode_homo}
\beq{
 \label{eq:heat_avg}
 x^{k+1} = x^k - \f{\eta}{L}\ \sqbrac{\sum_{i=1}^L \grad f_{\trm{mb}}(x^k+\e^i)}
}
where $\e^i$ are Gaussian random variables with zero mean and variance $\g\ I$. Note that we have implemented the convolution in~\eqref{eq:heat_ode_homo} as an average over Gaussian perturbations of the parameters $x$. We again set $L=20$ for our experiments in Sec.~\ref{s:expts}.

\subsection{Momentum}
\label{ss:momentum}

Momentum is a popular technique for accelerating the convergence of SGD~\citep{nesterov1983method}. We employ this in our experiments by maintaining an auxiliary variable $z^k$ which intuitively corresponds to the velocity of $x^k$. The $x^k$ update in~\eqref{eq:esgd_discrete_x} is modified to be
\[
 \aed{
  x^{k+1} &= z^k -\eta\ \grad f_{\trm{mb}}(x)(z^k - \mu^{k})\\
  z^{k+1} &= x^k + \d\ \rbrac{x^k - x^{k-L}}
 }
\]
if $k/L$ is an integer; $x^k$ and $z^k$ are left unchanged otherwise. The other updates for $x^k$ in~\eqref{eq:hj_discrete} and~\eqref{eq:heat_avg} are modified similarly. The momentum parameter $\d=0.9$ is fixed for all experiments.

\subsection{Tuning \texorpdfstring{$\g$}{g}}
Theorem~\ref{thm:improvement} gives a rigorous justification to the technique of `scoping'' where we set $\g$ to a large value and reduce it as training progresses, indeed, as Figure~\ref{fig:smoothing} shows, a larger $\g$ leads to a smoother loss. Scoping reduces the smoothing effect and as $\g \to 0$, we have $f_\g(x) \to f(x)$. For the updates in Sec.~\ref{ss:esgd_algorithmic},~\ref{ss:hj_algorithmic} and~\ref{ss:heat_algorithmic}, we update $\g$ every time $k/L$ is an integer using
\[
 \g = \g_0\ (1-\g_1)^{k/L};
\]
we pick $\g_0 \in [10^{-4},\ 10^{-1}]$ so as to obtain the best validation error while $\g_1$ is fixed to $10^{-3}$ for all experiments. We have obtained similar results by normalizing $\g_0$ by the dimensionality of $x$, i.e., if we set $\g_0 := \g_0/n$ this parameter becomes insensitive to different datasets and neural networks.

\section{Empirical validation}
\label{s:expts}

\begin{wrapfigure}{r}{0.2\textwidth}
\centering
\includegraphics[width=0.2\textwidth]{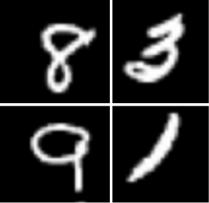}
\caption{\footnotesize MNIST}
\label{fig:mnist_data}
\vspace{0.1in}
\includegraphics[width=0.2\textwidth]{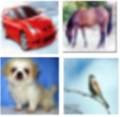}
\caption{\footnotesize CIFAR-10}
\label{fig:cifar_data}
\vspace*{-0.3in}
\end{wrapfigure}

We now discuss experimental results on deep neural networks that demonstrate that the PDE methods considered in this article achieve good regularization, aid optimization and lead to improved classification on modern datasets.

\subsection{Setup for deep networks}
\label{ss:setup}

In machine learning, one typically splits the data into three sets: (i) the training set $D$, used to compute the optimization objective (empirical loss $f$), (ii) the validation set, used to tune parameters of the optimization process which are not a part of the function $f$, e.g., mini-batch size, step-size, number of iterations etc., and, (iii) a test set, used to quantify generalization as measured by the empirical loss on previously-unseen (sequestered) data. However, for the benchmark datasets considered here, it is customary in the literature to not use a separate test set. Instead, one typically reports test error on the validation set itself; for the sake of enabling direct comparison of numerical values, we follow the same practice here.

We use two standard computer vision datasets for the task of image classification in our experiments. The MNIST dataset~\citep{lecun1998gradient} contains $70,000$ gray-scale images of size $28\times 28$, each image portraying a hand-written digit between $0$ to $9$. This dataset is split into $60,000$ training images and $10,000$ images in the validation set. The CIFAR-10 dataset~\citep{krizhevsky2009learning} contains $60,000$ RGB images of size $32 \times 32$ of $10$ objects (aeroplane, automobile, bird, cat, deer, dog, frog, horse, ship and truck). The images are split as $50,000$ training images and $10,000$ validation images. Figure~\ref{fig:mnist_data} and Figure~\ref{fig:cifar_data} show a few example images from these datasets.

\subsubsection{Pre-processing}

We do not perform any pre-processing for the MNIST dataset. For CIFAR-10, we perform a global contrast normalization~\citep{coates2010analysis} followed by a ZCA whitening transform (sometimes also called the ``Mahalanobis transformation'')~\citep{krizhevsky2009learning}. We use ZCA as opposed to principal component analysis (PCA) for data whitening because the former preserves visual characteristics of natural images unlike the latter; this is beneficial for convolutional neural networks. We do not perform any data augmentation, i.e., transformation of input images by mirror-flips, translations, or other transformations.

\subsubsection{Loss function}
We defined the zero-one loss $f_i(x)$ for a binary classification problem in Sec.~\ref{ss:deep_networks}. For classification tasks with $K$ classes, it is common practice to construct a network with $K$ distinct outputs. The output $y(x;\ \xi^i)$ is thus a vector of length $K$ that is further normalized to sum up to $1$ using the $\softmax$ operation defined in Sec.~\ref{ss:mnist}. The cross-entropy loss typically used for optimizing deep networks is defined as:
\[
 f_i(x) := - \sum_{k=1}^K\ \ind_{\cbrac{y^i=k}}\ \log y(x;\ \xi^i)_{k}.
\]
We use the cross-entropy loss for all our experiments.

\subsubsection{Training procedure}
We run the following algorithms for each neural network, their algorithmic details are given in Sec.~\ref{s:algorithmic_details}. The ordering below corresponds to the ordering in the figures.
\ilist{
 \item \textbf{Entropy-SGD:} the algorithm of~\citet{chaudhari2016entropy} described in Sec.~\ref{ss:esgd_algorithmic},
 \item \textbf{HEAT:} smoothing by the heat equation described in Sec.~\ref{ss:heat_algorithmic},
 \item \textbf{HJ:} updates for the non-viscous Hamilton-Jacobi equation described in Sec.~\ref{ss:hj_algorithmic},
 \item \textbf{SGD:} stochastic gradient descent given by~\eqref{eq:sgd},
}
All the above algorithms are stochastic in nature; in particular, they sample a mini-batch of images randomly from the training set at each iteration before computing the gradient. A significant difference between \textbf{Entropy-SGD} and \textbf{HJ} is that the former adds additional noise, while the latter only uses the intrinsic noise from the mini-batch gradients.
We therefore report and plot the mean and standard deviation across $6$ distinct random seeds for each experiment. The mini-batch size ($\abs{i_t}$ in~\eqref{eq:sgd}) is fixed at $128$. An ``epoch'' is defined to be one pass over the entire dataset, for instance, it consists of $\ceil{60,000/128} = 469$ iterations of~\eqref{eq:sgd} for MNIST with each mini-batch consisting of $128$ different images.

We use SGD as a baseline for comparing the performance of the above algorithms in all our experiments.

\begin{remark}[Number of gradient evaluations]
A deep network evaluates the average gradient over an entire mini-batch of $128$ images in one single back-propagation operation~\citep{rumelhart1988learning}. We define the number of back-props per weight update by $L$. SGD uses one back-prop per update (i.e., $L=1$) while we use $L=20$ back-props for Entropy-SGD (cf.\@ Sec.~\ref{ss:esgd_algorithmic}) and the heat equation (cf.\@ Sec.~\ref{ss:heat_algorithmic}) and $L=5$ back-props for the non-viscous HJ equation (cf.\@ Sec.~\ref{ss:hj_algorithmic}). For each of the algorithms, we plot the training loss or the validation error against the number of ``effective epochs'', i.e., the number of epochs is multiplied by $L$ which gives us a uniform scale to measure the computational efficiency of various algorithms. This is a direct measure of the wall-clock time and is agnostic to the underlying hardware and software implementation.
\end{remark}

\subsection{MNIST}
\label{ss:mnist}

\subsubsection{Fully-connected network}
\label{sss:mnistfc}

Consider a ``fully-connected'' network on MNIST
\[
 \mnistfc:\ \data_{784} \to \drop_{0.2} \to \underbrace{\fc_{1024} \to \drop_{0.2}}_{\times 2} \to \fc_{10} \to \softmax.
\]
The $\data$ layer reshapes each MNIST image as a vector of size $784$. The notation $\fc_{d}$ denotes a ``fully connected'' dense matrix $x^k$ in~\eqref{eq:yhat} with $d$ output dimensions, followed by the ReLU nonlinearity (cf.\@ Sec.~\ref{ss:deep_networks}) and an operation known as batch normalization~\citep{ioffe2015batch} which whitens the output of each layer by subtracting the mean across a mini-batch and dividing by the standard deviation. The notation $\drop_{p}$ denotes the dropout layer~\citep{srivastava2014dropout} which randomly sets a fraction $p$ of the weights to zero at each iteration; this is a popular regularization technique in deep learning, see~\citet{kingma2015variational,achilleS2017information} for a Bayesian perspective. For $10$ classes in the classification problem, we create an output vector of length $10$ in the last $\fc$ layer followed by $\softmax$ which picks the largest element of this vector, which is interpreted as the output (or ``prediction'') by this network. 
The $\mnistfc$ network has $n=1.86$ million parameters.

As Figure~\ref{fig:mnistfc_valid} and Table~\ref{tab:expts} show, the final validation error for all algorithms is quite similar. The convergence rate varies, in particular, Entropy-SGD converges fastest in this case. Note that $\mnistfc$ is a small network and the difference in the performance of the above algorithms, e.g., $1.08\ \%$ for Entropy-SGD versus $1.17\ \%$ for HJ, is minor.

\subsubsection{LeNet}
\label{sss:lenet}

Our second network for MNIST is a convolutional neural network (CNN) denoted as follows:
\[
 \lenet:\ \data_{28\times 28}\ \to \convolution_{20, 5, 3} \to \drop_{0.25} \to \convolution_{50, 5, 2} \to \drop_{0.25} \to \fc_{500} \to \drop_{0.25} \to \fc_{10} \to \softmax.
\]
The notation $\convolution_{c,k,m}$ denotes a 2D convolutional layer with $c$ output channels, each of which is the sum of a channel-wise convolution operation on the input using a learnable kernel of size $k\times k$. It further adds ReLU nonlinearity, batch normalization and an operation known as max-pooling which down-samples the image by picking the maximum over a patch of size $m\times m$. Convolutional networks typically perform much better than fully-connected ones on image classification task despite having fewer parameters, $\lenet$ has only $n=131,220$ parameters.

\begin{figure}[htp!]
\centering
 \begin{subfigure}[t]{0.4\textwidth}
 \centering
 \includegraphics[width=\textwidth]{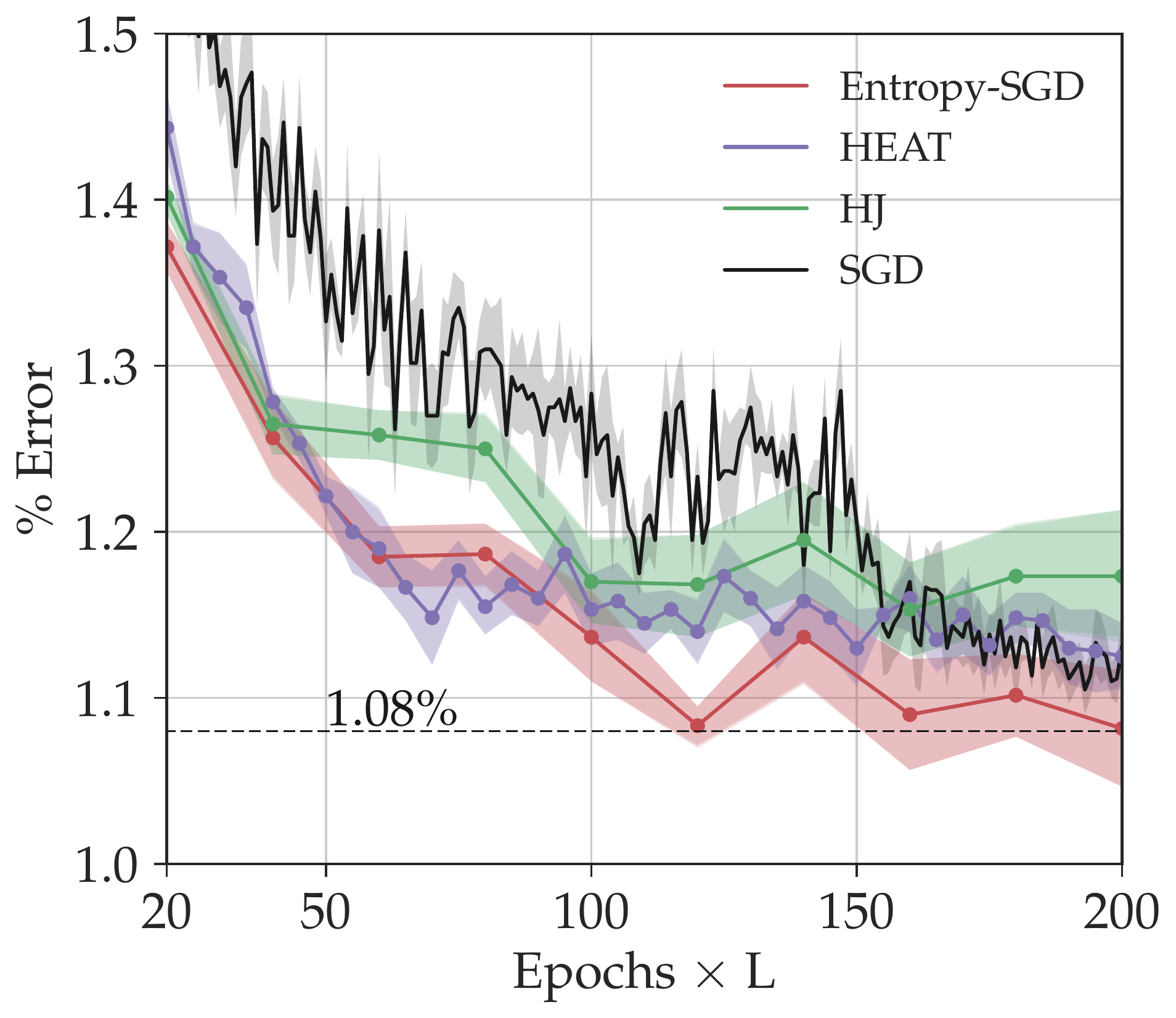}
 \caption{\small $\mnistfc$: Validation error}
 \label{fig:mnistfc_valid}
 \end{subfigure}
 \hspace{0.15in}
 \begin{subfigure}[t]{0.4\textwidth}
 \centering
 \includegraphics[width=\textwidth]{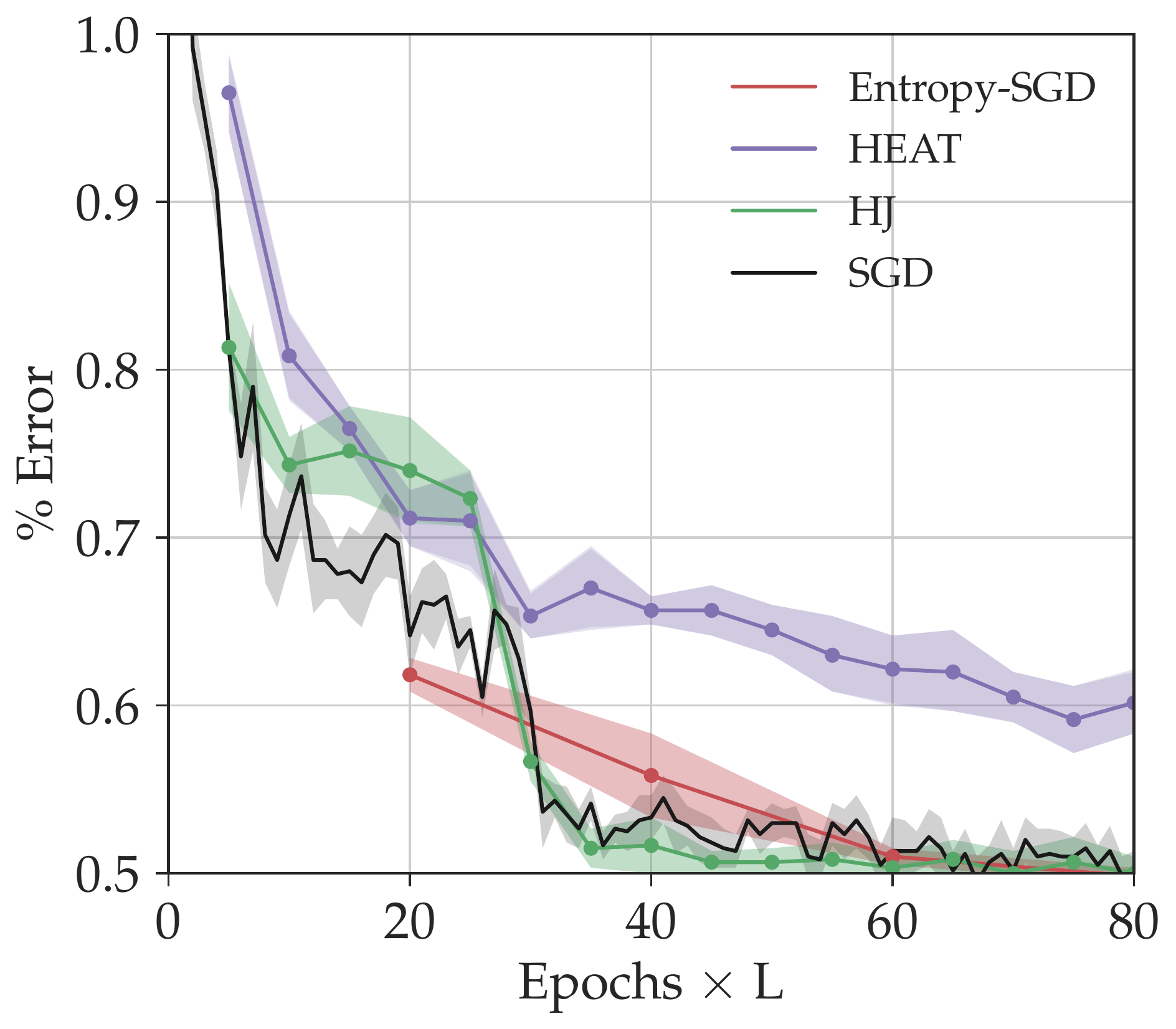}
 \caption{\small $\lenet$: Validation error}
 \label{fig:lenet_valid}
 \end{subfigure}
\caption{\small $\mnistfc$ and $\lenet$ on MNIST (best seen in color)}
\label{fig:mnist}
\end{figure}
The results for $\lenet$ are described in Figure~\ref{fig:lenet_valid} and Table~\ref{tab:expts}. This is a convolutional neural network and performs better than $\mnistfc$ on the same dataset. The final validation error is very similar for all algorithms at $0.50\ \%$ with the exception of the heat equation which only reaches $0.59\ \%$. We can also see that the other algorithms converge quickly, in about half the number of effective epochs as that of $\mnistfc$.

\subsection{CIFAR}
\label{ss:cifar}

The CIFAR-10 dataset is more complex than MNIST and fully-connected networks typically perform poorly. We will hence employ a convolutional network for this dataset. We use the All-CNN-C architecture introduced by~\citet{springenberg2014striving} and add batch-normalization:
\[
 \allcnn:\ \data_{3\times 32\times 32}\ \to \drop_{0.2} \to \block_{96, 3} \to \block_{192, 3} \to \underbrace{\convolution_{192,3}}_{\times 2} \to \convolution_{10} \to \meanpool_{10} \to \softmax.
\]
where
\[
 \block_{d,3}:\ \convolution_{d,3,1} \to \convolution_{d,3,1}\ \to \convolution^*_{d,3,1} \to \drop_{0.5}.
\]
The final convolutional layer in the above $\block$ denoted by $\convolution^*$ is different from others; while they perform convolution at every pixel otherwise known as a ``stride'' of $1$ pixel, $\convolution^*$ on the other hand uses a stride of $2$; this results in the image being downsampled by a factor of two. Note that $\convolution_{c,k,m}$ with $m=1$ does not result in any downsampling. Max-pooling usually results in a drastic reduction of the image size and by replacing it with a strided convolution, $\allcnn$ achieves improved performance with much fewer parameters than many other networks on CIFAR-10. The final layer denoted as $\meanpool$ takes an input of size $10 \times 8 \times 8$ and computes the spatial average to obtain an output vector of size $10$. This network has $n=1.67$ million parameters.

\begin{figure}[htp!]
\centering
 \begin{subfigure}[t]{0.4\textwidth}
 \centering
 \includegraphics[width=\textwidth]{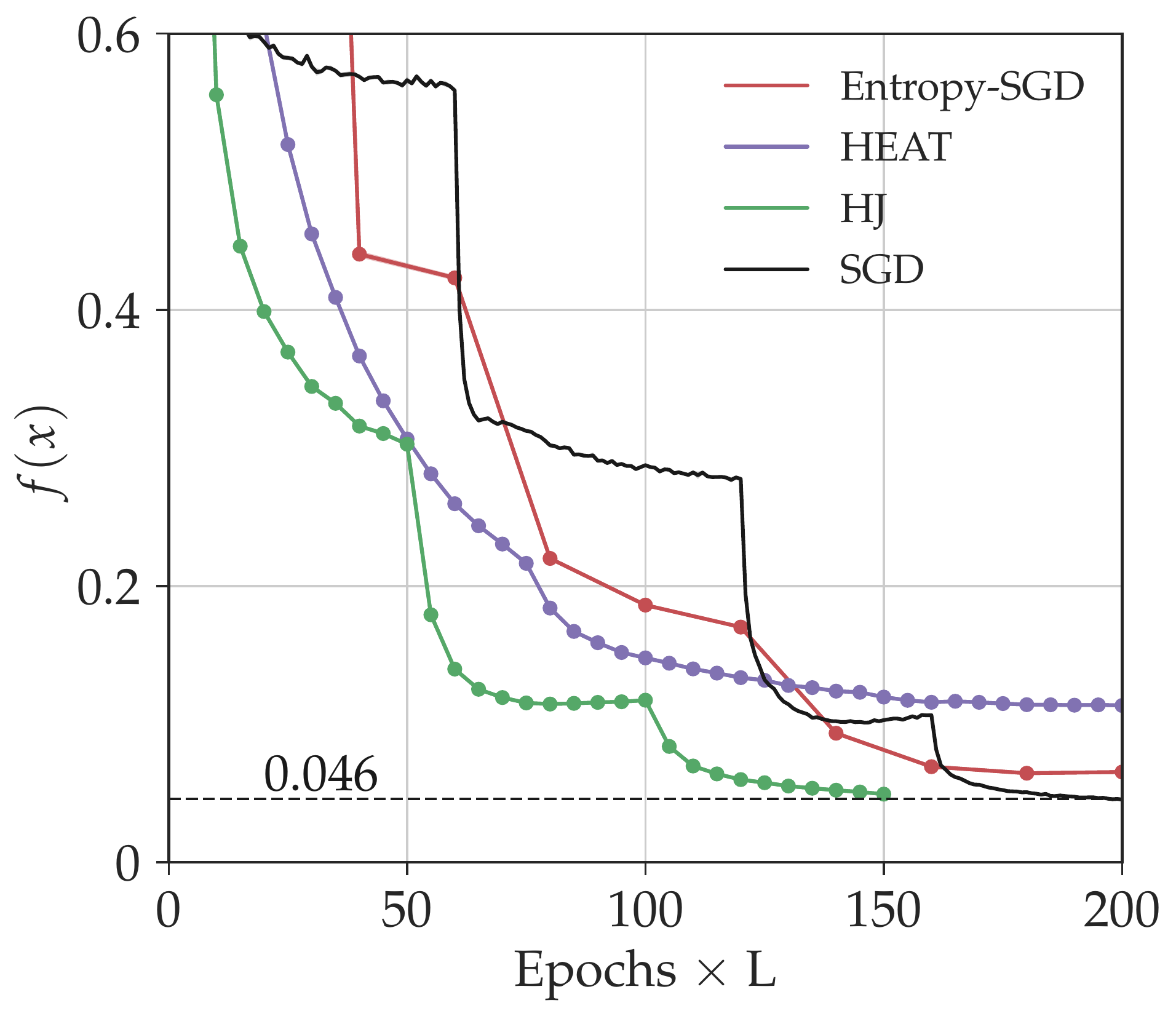}
 \caption{\small $\allcnn$: Training loss}
 \label{fig:allcnn_loss}
 \end{subfigure}
 \hspace{0.15in}
 \begin{subfigure}[t]{0.4\textwidth}
 \centering
 \includegraphics[width=\textwidth]{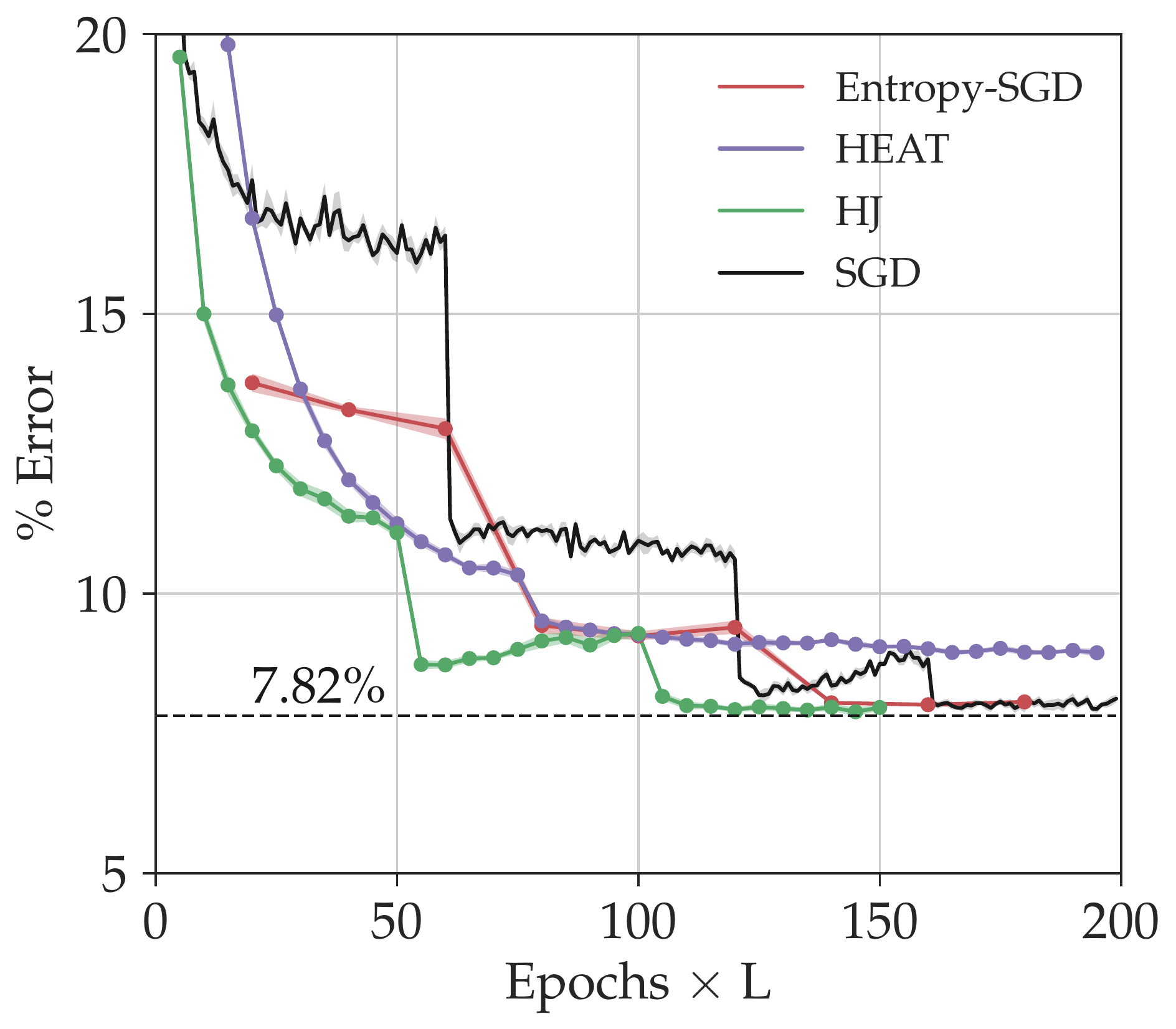}
 \caption{\small $\allcnn$: Validation error}
 \label{fig:allcnn_valid}
 \end{subfigure}
\caption{\small Training loss and validation error on CIFAR-10 (best seen in color)}
\label{fig:cifar}
\end{figure}

Figure~\ref{fig:allcnn_loss} and~\ref{fig:allcnn_valid} show the training loss and validation error for the $\allcnn$ network on the CIFAR-10 dataset. The Hamilton-Jacobi equation (HJ) obtains a validation error of $7.89\ \%$ in $145$ epochs and thus performs best among the algorithms tested here; it also has the lowest training cross-entropy loss of $0.046$. Note that both HJ and Entropy-SGD converge faster than SGD. The heat equation performs again poorly on this dataset and has a much higher validation error than others ($9.04 \%$). This is consistent with our discussion in Sec.~\ref{ss:heat_vs_hjb} which suggests that the viscous or non-viscous HJ equations result in better smoothing than the heat equation.

{
\setlength{\heavyrulewidth}{1.5pt}
\renewcommand{\arraystretch}{1.5}
\begin{table}[H]
\centering
\large
\resizebox{0.8 \columnwidth}{!}
{
\begin{tabular}{p{4cm} | c | c | c | c}
\toprule
 \rowcolor{gray!15} Model & Entropy-SGD & HEAT & HJ & SGD\\
\toprule
  $\mnistfc$ & ${\bf 1.08 \pm 0.02\ @\ 120}$
   & $1.13 \pm 0.02\ @\ 200$
   & $1.17 \pm 0.04\ @\ 200$
   & $1.10 \pm 0.01\ @\ 194$\\

 \rowcolor{gray!15} $\lenet$ & $0.5 \pm 0.01\ @\ 80$
   & $0.59 \pm 0.02\ @\ 75$
   & ${\bf 0.5 \pm 0.01\ @\ 70}$
   & $0.5 \pm 0.02\ @\ 67$\\

  $\allcnn$ & $7.96 \pm 0.05\ @\ 160$
   & $9.04 \pm 0.04\ @\ 150$
   & ${\bf 7.89 \pm 0.07\ @\ 145}$
   & $7.94 \pm 0.06\ @\ 195$\\
\bottomrule
\end{tabular}
}
\caption{\small Summary of experimental results: Validation error (\%) @ Effective epochs}
\label{tab:expts}
\end{table}

\section{Discussion}
\label{s:discussion}

Our results apply nonlinear PDEs, stochastic optimal control, and stochastic homogenization to the analysis of two recent and effective algorithms for the optimization of neural networks. Our analysis also contributes to devising improved algorithms.

We replaced the standard stochastic gradient descent (SGD) algorithm for the function $f(x)$, with SGD on the two variable function $H(x,y;\g) = f(y) + \g^{-1}|x-y|^2/2$, along with a small parameter $\e$~\eqref{eq:le_sde_homo}. Using the new system, in the limit $\e\to 0$,
we can provably, and quantitatively, improve the expected value of the original loss function. The effectiveness of our algorithm comes from the connection, via homogenization, of system of SDEs to the gradient of the function $u(x,t)$ which is the solution of the~\eqref{eq:vhj} PDE with initial data $f(x)$. The function $H$ is more convex in $y$ that the original function $f$. The convexity of $H$ in $y$ is related to the exponentially fast convergence of the dynamics, a connection explained by the celebrated gradient flow interpretation of Fokker-Planck equation \cite{jordan1998variational}.
Ergodicity of the dynamics for \eqref{eq:le_sde_homo} makes clear that the local entropy SGD is equivalent to the influential distributed algorithm, Elastic-SGD \cite{zhang2015deep}. These insights may lead to improved distributed algorithms, which are presently of great interest.

On a practical level, the large number of hyper-parameters involved in training deep neural networks is very costly, in terms of both human effort and computer time. Our analysis lead to better understanding of the parameters involved in the algorithms, and provides an insight into the choice of hyper-parameters for these methods. In particular (i) the parameter $\g$ is now understood as the time $t$, in the PDE \eqref{eq:vhj} (ii) scoping of $\g$, which was seen as a heuristic, is now rigorously justified, (iii) we set the extrinsic noise parameter $\b^{-1}_{\trm{ex}}$ to zero, resulting in a simpler, more effective algorithm, and (iv) we now understand that below a critical value of $\g$ (which is determined by the requirement that $H$ be convex in $y$), the convergence of the dynamics in \eqref{eq:le_sde_homo} to the invariant measure is exponentially fast.

Conceptually, while simulated annealing and related methods work by modulating the level of noise in the dynamics, our algorithm works by modulating the smoothness of the underlying loss function. Moreover, while most algorithms used in deep learning derive their motivation from the literature on convex optimization, the algorithms we have presented here are specialized to non-convex loss functions and have been shown to perform well on these problems both in theory and in practice.

\section{Acknowledgments}
\label{sec:ack}

AO is supported by a grant from the Simons Foundation (395980); PC and SS by ONR N000141712072, AFOSR FA95501510229 and ARO W911NF151056466731CS; SO by ONR N000141410683, N000141210838, N000141712162 and DOE DE-SC00183838. AO would like to thank the hospitality of the UCLA mathematics department where this work was completed.

{
\footnotesize
\bibliographystyle{apalike}
\bibliography{chaudhari.oberman.ea}
}

\end{document}